\documentclass[12pt,journal,onecolumn,draftclsnofoot]{IEEEtran}
\usepackage{epsfig}
\usepackage{times}
\usepackage{float}
\usepackage{afterpage}
\usepackage{amsmath}
\usepackage{mathrsfs}
\usepackage{amstext}
\usepackage{amssymb,bm}
\usepackage{latexsym}
\usepackage{color}
\usepackage{graphicx}
\usepackage{amsmath}
\usepackage{amsthm}
\usepackage{tikz}
\usepackage{tikzscale}
\usepackage{graphicx}
\usepackage[center]{caption}
\usepackage{pstricks}
\usepackage{makecell}
\usepackage{subcaption}
\usepackage{booktabs}
\usepackage{multicol}
\usepackage{lipsum}
\usepackage{dblfloatfix}
\usepackage{algorithm} 
\usepackage[noend]{algpseudocode} 
\usepackage[normalem]{ulem}
\usepackage{enumerate}
\usepackage{bbm}
\usepackage{dsfont}
\usepackage{tikz}
\usepackage[shortlabels]{enumitem}

\def \bx{\mathbf{x}}
\def \bz{\mathbf{z}}

\newcommand{\mbb}{\mathbb}
\newcommand{\mb}{\mathbf}
\newcommand{\mcal}{\mathcal}

\DeclareMathOperator*{\argmax}{arg\,max}
\DeclareMathOperator*{\argmin}{arg\,min}

\newtheorem{theorem}{Theorem}
\newtheorem{thm}{Theorem}

\newtheorem{lem}[thm]{Lemma}

\newtheorem{rem}{Remark}
\newtheorem{defin}{Definition}
\newtheorem*{defin*}{Definition}

\algnewcommand\algorithmicforeach{\textbf{for each}}

\algdef{S}[FOR]{ForEach}[1]{\algorithmicforeach\ #1\ \algorithmicdo}

\usepackage{xcolor}
\newcommand*\circled[1]{\tikz[baseline=(char.base)]{
		\node[shape=circle,draw,inner sep=1pt] (char) {#1};}}

\input{widebar_hack.tex}

\begin{document}

\title{On Distributed Quantization for Classification} 

\author{\IEEEauthorblockN{Osama A. Hanna$^\dagger$, Yahya H. Ezzeldin$^\dagger$, Tara Sadjadpour$^\dagger$,\\ Christina Fragouli$^\dagger$ and Suhas Diggavi$^\dagger$ \\ 
$^\dagger$University of California, Los Angeles\\
Email:\{ohanna, yezzeldin, tsadja, christina.fragouli, suhasdiggavi\}@ucla.edu}
\thanks{
This work was supported in part by NSF grants 1514531, 1824568 and  by UC-NL grant LFR-18-548554.
}
}
\maketitle

\begin{abstract}
  We consider the problem of distributed feature quantization, where the goal is to enable a pretrained classifier 
  at a central node to carry out its classification on features that are gathered from distributed nodes through communication constrained channels.  We propose the design of distributed quantization schemes specifically tailored to the classification task: unlike quantization schemes that help the central node reconstruct the original signal as accurately as possible, our focus is not reconstruction accuracy, but instead correct classification.
Our work does not make any apriori distributional assumptions on the data, but instead uses training data for the quantizer design.
Our main contributions include: we prove  NP-hardness of finding
optimal quantizers in the  general case; we design an optimal scheme for a special case; we propose quantization algorithms, that leverage discrete neural representations and training data, and can be designed in polynomial-time  for any number of features, any number of classes, and arbitrary division of features across the distributed nodes. We find that tailoring the quantizers to the classification task can offer significant savings: as compared to alternatives, we can achieve more than a factor of two reduction in terms of the number of bits communicated, for the same classification accuracy.

\end{abstract}

\section{Introduction}
Quantization forms the core of almost all lossy data-compression
algorithms, and is widely used to reduce the number of bits required
for storage and communication.  These schemes optimize a
rate-distortion trade-off, where the goal is to represent data using a
limited number of bits as precisely as possible. Instead, in this
work, we propose  distributed quantization
schemes tailored to data that are going to be used for
classification. That is, we explore the design of distributed
quantizers for a rate-classification error trade-off: our quantization
schemes are not optimized for reconstruction accuracy, but instead
correct classification.

Figure~\ref{Fig1} illustrates the difference
between the two aforementioned approaches. For a given number of
bits, we create a corresponding number of quantization regions in the space ($3$ bits
per feature = $2^3.2^3$ = $64$ regions in our example). Intuitively, for
data reconstruction, we want to more finely represent the regions of
high signal concentration (Figure~\ref{Fig1}(b)); for classification,
we want to more finely represent areas closer to the classification
boundary where errors may happen (Figure~\ref{Fig1}(c)).
\begin{figure}
	\centering
    \begin{subfigure}[b]{0.3\linewidth}
        \includegraphics[draft=false,width=\textwidth]{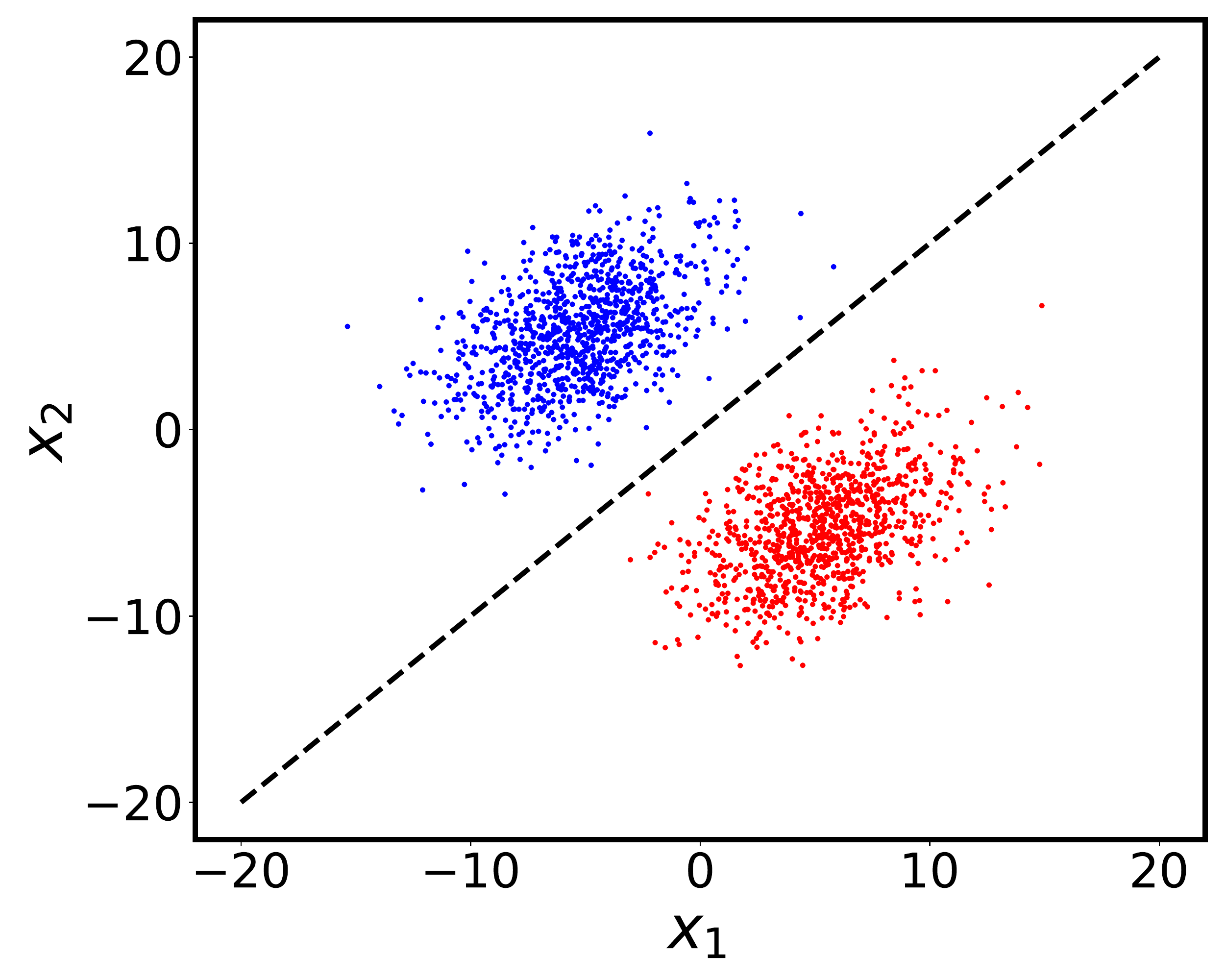}
        \caption{Unquantized}
        \label{fig:example1a}
    \end{subfigure}
~
    \begin{subfigure}[b]{0.3\linewidth}
        \includegraphics[draft=false,width=\textwidth]{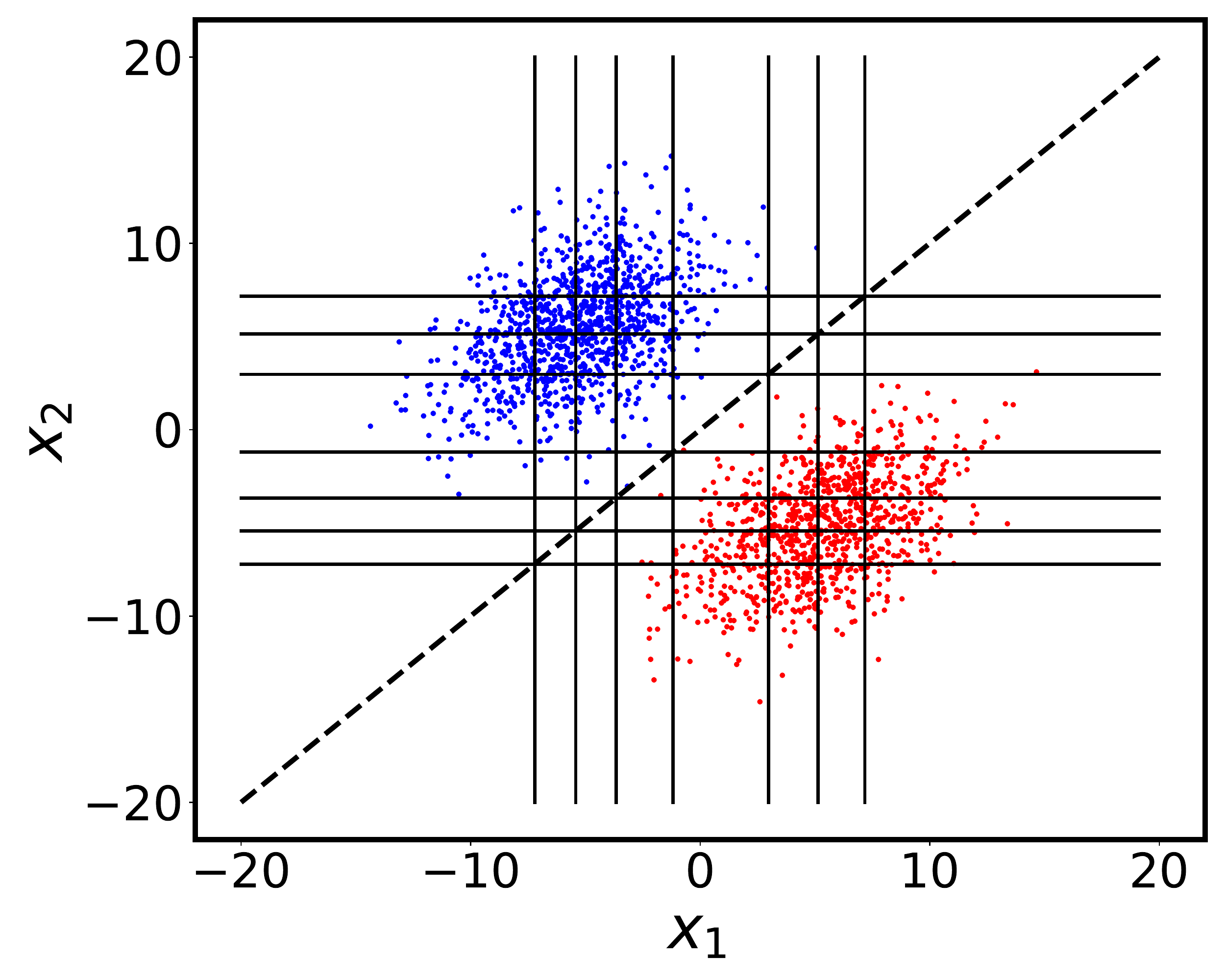}
        \caption{Uniform}
        \label{fig:example1b}
    \end{subfigure}
~
    \begin{subfigure}[b]{0.3\linewidth}
        \includegraphics[draft=false,width=\textwidth]{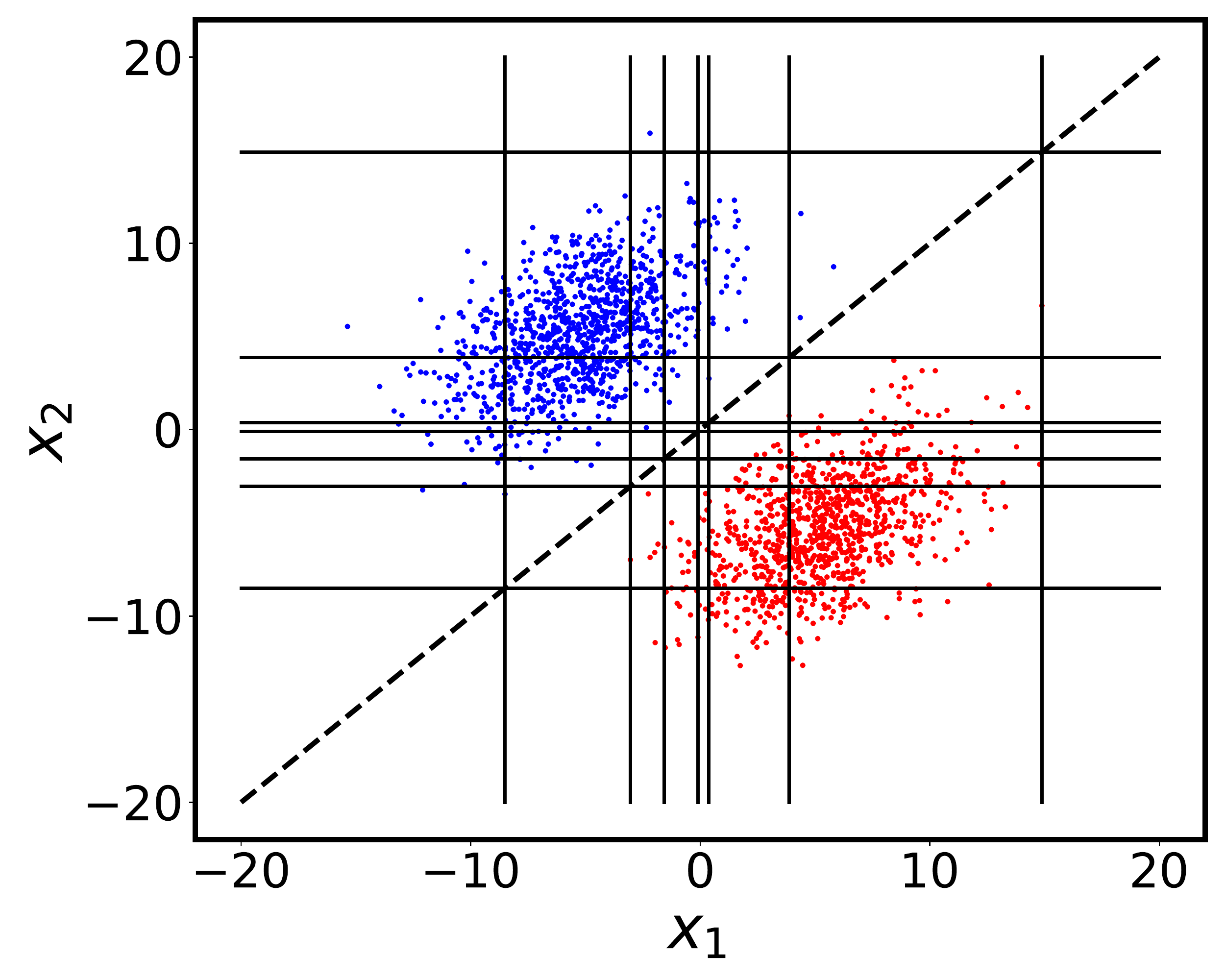}
        \caption{Classification-aware}
        \label{fig:example1c}
    \end{subfigure}
    \caption{Example with 2 features and 3 bits/feature: {\bf (a)} Unquantized points separated by a linear classifier; {\bf (b)} Uniform measure quantization;  {\bf (c)}  Classification-aware quantization.}
    \label{Fig1}
\end{figure}

In our work, we aim to design quantizers for the following generic
scenario. A central entity has access to a pretrained subdifferentiable 
classifier (\emph{e.g.,} a multilayer perceptron
- MLP \cite{DLbook}) and wishes to apply that classifier on data
features collected at $K$ distributed sensor nodes. The communication
between the sensors and the central entity comes at a cost (is rate
limited), and thus it is expensive to send the measured features with
full precision.  Instead, each node employs a \emph{distributed
  single-shot} quantizer\footnote{Single-shot quantization means that
  we do not collect samples over time and jointly quantize
  them. Therefore a set of \emph{local} features observed at a node
  are quantized together whenever observed; motivated by delay
  requirement for classification.}, independently from other nodes, in
order to encode its measurements into bit representations that can be
sent to the central entity efficiently as soon as sensed.  We
emphasize that we do not make any apriori distributional assumptions
on the data, as is common in many learning scenarios. Moreover, the
data may be heterogeneous, from unknown composite distributions
(\emph{e.g.,} multimodal observations of sensors that capture video,
sound, and radar signals).  We simply use training samples from the
data to design the quantizers.

This scenario is motivated by many machine learning applications,
that include wireless cyberphysical systems, immersive environments
and supported health.  For example, in brain-to-computer interface
applications, multiple electrodes are placed around the brain to
capture brain signals which are used collectively as features to
classify in what direction a person is trying to move his
hand~\cite{lebedev2006brain}.  Such features need to be quantized at
the sensor peripheral nodes, and communicated through rate constrained
channels to a central node for processing, so that classification (and
decisions based on it) can be done within a reasonable time of
sensing.


 Our assumption of a pretrained classifier is motivated by the
following practical considerations: (i) we may not know the
communication channel constraints when designing the classifier and we
may want to use the same classifier over systems with different
communication channels;
(ii) we may not have access to the data used to train the classifier
(\emph{e.g.,} we use a pretrained classifier from cloud services such as
Google Cloud~\cite{googleCloud} or Clarifai~\cite{clarifai}), but are
able to personalize the quantizers leveraging locally available data.

\noindent{\bf Contributions.}  We begin our formal study of the
problem by proving that in general, it is NP-hard\footnote{This
  hardness is in terms of the problem parametrization, \emph{e.g.,}
  number of training points and the number of features.} to design an
optimal distributed quantization system tailored for classification of
a number of given data points. We also show that the problem is hard
to approximate, therefore motivating alternate approaches to the
design, and empirical evaluations of proposed techniques.


Given the difficulty of designing the optimal
quantization system, we propose a data-based greedy quantization
boundary insertion strategy, which we term $\texttt{GBI}$, which can
be used for any type of classifier, any number of features, any number
of classes, and arbitrary division of features across the distributed
nodes. Operationally, $\texttt{GBI}$ creates rectangular quantization
regions by greedily deciding how to divide the training data along each feature.  We demonstrate that $\texttt{GBI}$ has quadratic complexity in
the number of training samples and linear in the number of features.

To further reduce complexity and capture richer quantization
boundaries (beyond rectangular), we propose a (deep) learning based
approach to design our quantizers that makes use of the
subdifferentiable nature of the classifier employed by the central
node.  This is inspired by the recent success of learning discrete
latent variables~\cite{VQ-VAE,VQ-VAE2}, joint source channel
coding~\cite{JSCC} and discrete representations for image
compression~\cite{theis2017lossy,balle2016end,van2016pixel}.  Our
design can be understood as a \emph{distributed} discrete neural
representation optimized for classification.  We leverage the
$\texttt{GBI}$ algorithm by making it a module within the discrete
neural representation.

Through numerical evaluation,
we show  that for the same representation
budget (number of bits available at sensor nodes for each
measurement), we can achieve four folds gains in classification accuracy compared to approaches that try to
learn discrete representations aimed directly at reconstruction.

Our main contributions can thus be summarized as follows:
\begin{itemize}
  \item We prove the NP-hardness and hardness to approximation for
    designing optimal distributed quantizers for classification.
    \item We design optimal quantizers for linearly separable data and two features under some structural restrictions. 
  \item We propose a polynomial-complexity greedy quantization
    algorithm,  $\texttt{GBI}$, optimized for classification, that can be used for any number of features and any classifier.
  \item We propose a novel distributed discrete neural representation
    for classification, which can also be combined with
    $\texttt{GBI}$.
  \item When compared with approaches for data reconstruction, we
    demonstrate benefits of 50\% gain in terms of
    classification accuracy for our proposed quantization approaches,
    on an sEMG dataset and 300\% improvement on the CIFAR10 image dataset.
\end{itemize}	  

\noindent{\bf Paper Organization.} Section~\ref{sec:related_work},
reviews related work; Section~\ref{sec:model} develops the notation
and problem framework; Section~\ref{sec:complexity} proves the
NP-hardness results; Section~\ref{sec:GBI}, introduces the
$\texttt{GBI}$ algorithm; Section~\ref{sec:neural} proposes neural
representation schemes; and Section~ \ref{sec:experiments} presents
our numerical evaluation.





\section{Related Work}
\label{sec:related_work}
We will give representative examples of related literature to put our
work in context, with an organization around specific
approaches/problems.

{\bf Distributed detection and hypothesis
  testing.}  The problem studied in this paper is related to
distributed estimation and detection in communication-constrained
networks, extensively studied in the literature
(see~\cite{chamberland2007wireless,luo2005universal} and references
therein). Differently from our work, a common assumption is that
sensor measurements are independently distributed given the detection
hypothesis, and that these conditional distributions are
known. In~\cite{longo1990quantization}, scalar quantization for
distributed hypothesis testing was studied, using \emph{known}
conditional distribution of features.  In contrast to all these works,
we neither assume knowledge of the sensor measurements distribution,
nor do we make independence assumptions.

The information theoretic study through error exponents where features
are observed at different nodes, is surveyed
in~\cite{amari1998statistical}. Here, differently from our single-shot
setup, an asymptotically long sequence of i.i.d. time samples, from a
fixed underlying (unknown) distribution, are jointly compressed to
distinguish between two hypotheses (\emph{e.g.,} testing for
independence).

There have also been several recent works in information theory and
machine learning on distributed probability estimation, property
testing and
simulation~\cite{han2018geometric,diakonikolas2017communication,acharya2018distributed}. These
works assume that each node observes all features, and has access to
independent samples from an unknown underlying distribution. Distinct
from this in our setup, each node observes subsets of (non-overlaping)
features, \emph{i.e.,} the observations at different nodes are not
identically distributed.

{\bf Scalar and Vector quantization.}  In
\cite{Poor-HVQ-88,sun2013distributed,misra2011distributed} and references therein,
a high-rate quantization theory is developed for computing
\emph{known} functions from distributed observations, where the source
distributions are known.  A framework for \emph{centralized}
quantizers used for classification can be designed using the learning
vector quantization (LVQ)
framework~\cite{kohonen1996lvq,sato1996generalized}, where a number of
prototype classified vectors are defined and updated to reduce the
misclassification error. In contrast, our problem requires decentralized
quantizatio; moreover, the Voronoi tessellation in LVQ may not be
decomposable into decision boundaries applicable by distributed
quantization.

{\bf Multi-terminal function computation.}  Rate-distortion literature
has considered several related problems, where asymptotically large
number of samples are jointly represented; moreover these problems assume that
distribution of the sources are known.  In the classical CEO
problem~\cite{berger1996ceo,viswanathan1997quadratic,oohama1998rate,prabhakaran2004rate},
a central node wishes to reconstruct a value from independently
corrupted versions measured at distributed sensors.  Distributed
compression for functional computation with distortion has been
studied
in~\cite{wagner2011distributed,krithivasan2009lattices,doshi2010functional}. Our
work focuses on single-shot quantization for apriori unknown source
distributions, without an explicit knowledge of the classifier
function.

{\bf Model Compression.} Quantization is also used in inference tasks
for model
compression~\cite{ZhouYGXC17,jacob2018quantization,wu2018training,RD-ModelCompICML19},
with the goal to simplify implementation and reduce storage. However,
differently from our work, the goal is to quantize the model operands
rather than focus on distributively quantizing the inputs to the
model.

{\bf Decision stumps.} A closely related algorithm that could be
adapted to use for feature quantization is
AdaBoost~\cite{freund1997decision,hastie2009multi} with decision
stumps. In this case, the majority rule on the decision stumps
naturally partitions (quantizes) the space based on the number of
stumps corresponding to each feature.  However, AdaBoost with decision
stumps will not necessarily be able to return viable quantizers in all
cases. For example, if we consider labeled data points with an XOR
pattern in $\mbb{R}^2$ (centered at [-1,-1], [-1,1],[1,-1] and [1,1]
as shown in Figure \ref{fig:xor_example},
Appendix~\ref{sec:xor_example}) then AdaBoost with stumps is not able
to represent the XOR pattern in its decision
regions~\cite{tu2014layered}. In contrast, it is not difficult to see
that two quantization boundaries at $x_1 = 0$ and $x_2 = 0$,
respectively, are enough to allow a good classifier to correctly
classify the quantized points.

{\bf Latent variable models.} Perhaps the closest approach to ours,
are those of learning latent representations for data
reconstruction. In variational autoencoders
(VAEs)~\cite{kingma2013auto,higgins2017beta,zhao2019infovae,mathieu2019disentangling},
a continuous latent representation space is learned from the inputs,
that can then be used to reconstruct inputs or generate new data that
follow the same distribution as the data in the training set.
In~\cite{VQ-VAE}, the authors present a new way of training VAEs to
learn discrete latent space representations, which naturally leads to
a compression algorithm, since continuous (or full-precision) inputs
can be mapped to discrete latent representations typically using fewer
bits. In~\cite{JSCC}, the authors also study the inference of discrete
latent variables for joint source and channel coding. In particular,
discrete latent variables are learned such that they can be used for
compression as in VQ-VAEs; they are also robust to transmission over
noisy discrete channels for reconstruction.

A main difference between these implementations and our setup is the
\emph{distributed} (decomposable) structure of our quantization
system. In addition, it is intuitive to expect that reconstruction may
not yield the best classification results; what is perceived by the
reconstruction loss as a good approximation of the image might be
inappropriate for a classifier as compared to the performance of a
classification tailored approach. We explore the latter point
empirically in Section~\ref{sec:experiments}.  Therefore, our work can
be thought of as an approach to distributed discrete (neural)
representation for classification.  Recently,~\cite{45903} presented
a variational approximation to the information bottleneck
method~\cite{tishby2000information} to design classifiers. However,
differently from our work, it assumed a centralized encoder and
continuous latent variables.

\section{Notation and Problem Formulation}
\label{sec:model}
\begin{figure}[h!]
	\centering
	\includegraphics[width=0.9\textwidth]{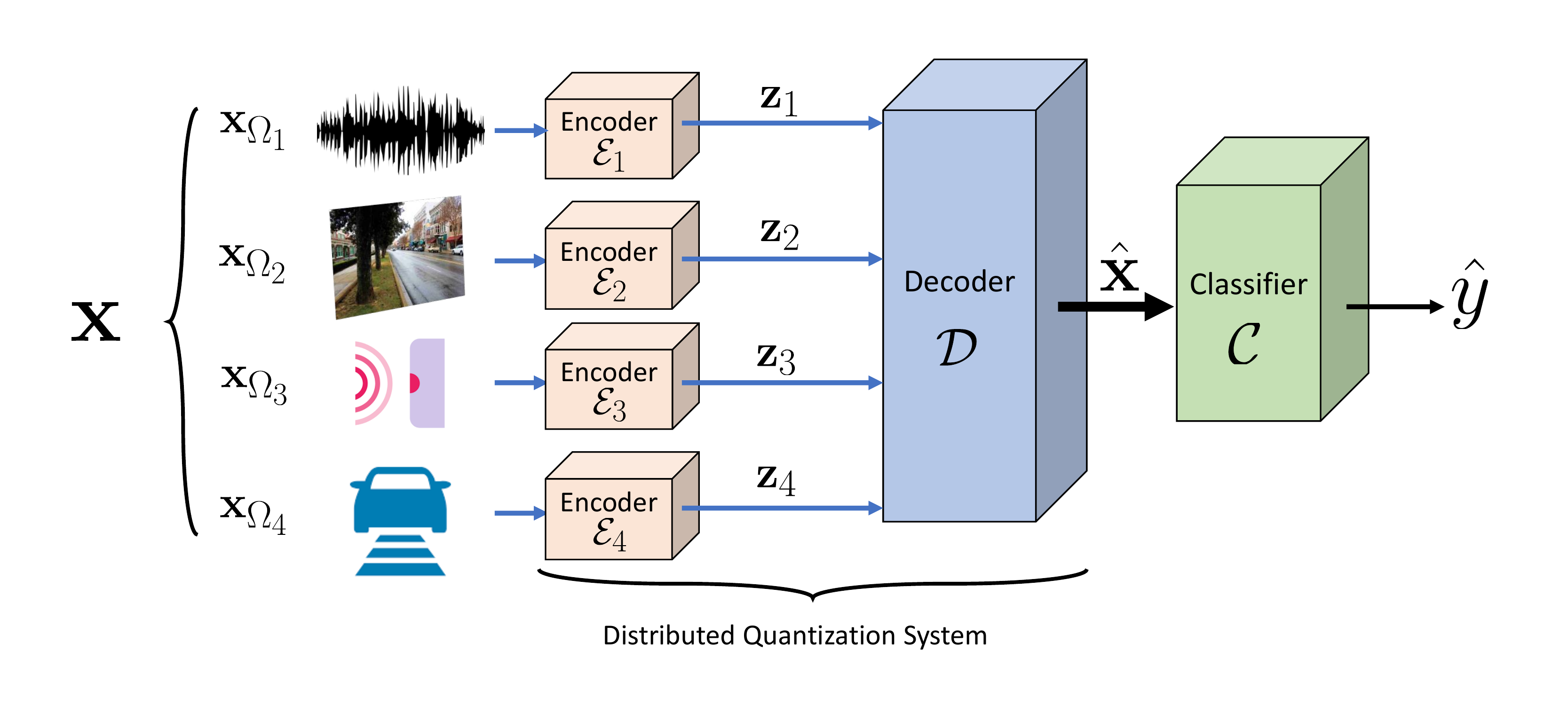}
	\vspace{-1em}
	\caption{\label{fig:model} An example for distributed quantization of features for classification with $K=4$ nodes.}
\end{figure}
Let $\mcal{X}^n$ be the $n$-dimensional space of possible input features (e.g., sensor measurements, images, text, etc.) and $\mcal{Y}$ be the set of possible classification classes over the space $\mcal{X}^n$. 
We use $\bx \in \mcal{X}^n$ to denote an $n$-dimensional input and $y(\mb{x}) \in \mcal{Y}$ to denote the class associated with $\mb{x}$. 

We consider a scenario where the features of $\mb{x}$ are not collected (sensed) all at the same entity but instead  at $K$ distributed nodes. An example is shown in Figure~\ref{fig:model} for $K=4$.
In particular, node $k\in [1:K]$ collects the vector of features $\mathbf{x}_{\Omega_k}$ indexed by a set $\Omega_k \subseteq [1:n]$. We assume that the index sets $\{\Omega_k\}_{k=1}^K$ are disjoint and their union is the set $[1:n]$. We also assume that the feature vector $\mathbf{x}$ is ordered such that $\mb{x} = [\mb{x}_{\Omega_1},\mb{x}_{\Omega_2},\cdots,\mb{x}_{\Omega_K}]$. We denote by $\mcal{X}_{\Omega_k}$ the space defined by the features indexed by $\Omega_k$, with
$\mcal{X}^n = \prod_{k=1}^K \mcal{X}_{\Omega_k}$.

A central node is interested in classifying the collected input features $\mb{x}$  using a pretrained subdifferentiable\footnote{By ``subdifferentiable'' classifier, we refer to a classifier that has non-trivial subgradient sets (i.e., not zero everywhere).} classifier $\mcal{C}(.)$ where
\begin{align}\label{def:C}
\mcal{C}(.) : \mcal{X}^n \to \mbb{R}^{|\mcal{Y}|}, 
\end{align}
and the output class label $\widehat{y}(\mb{x})$ is given by
\begin{align}\label{def:y_hat}
\widehat{y}(\mb{x}) = \argmax_{i \in [1:|\mcal{Y}|]} [\mcal{C}(\bx)]_i.
\end{align}
Note that $\widehat{y}(\mb{x})$ may be different than the true label $y(\mb{x})$.
With no communication constraints, node $k$ can perfectly convey  $\mathbf{x}_{\Omega_k}$ to the central node.
Instead, we assume that node $k$ is constrained to use $R_k$ bits (much less than full precision). That is, node $k$ uses a
quantizer/encoder $\mcal{E}_k$, that takes as input  $\mb{x}_{\Omega_k}$
and produces a discrete representation $\mb{z}_k$
from an alphabet $\mcal{M}_k$ of size at most $2^{R_k}$, with
\begin{align}
\label{eq:encoding}
\mb{z}_k &= \mcal{E}_k(\mathbf{x}_{\Omega_k}) : \mcal{X}_{\Omega_k} \to \mcal{M}_k ,\quad \forall k \in [1:K].
\end{align} 
Based on~\eqref{eq:encoding}, we will denote the preimage of $\mcal{E}_k$ as 
\begin{align}\label{eq:preimage}
\mathcal{E}_k^{-1}(\bz_k)=\{\bx_{\Omega_k}\in \mathcal{X}_{\Omega_k} |\mathcal{E}_k(\bx_{\Omega_k})=\bz_k\}.
\end{align}
Note that the computed $\bz_k$ depends only on $\bx_{\Omega_k}$, the features available at node $k$.
~
At the central node, in order to apply the pretrained classifier $\mcal{C}$, a decoder $\mcal{D}$ generates $\widehat{\mb{x}}\in \mathcal{X}^n$ from $\mb{z} = [\bz_1, \bz_2, \cdots, \bz_K]$ and uses it as the input to the classifier. 
The end-to-end operation, depicted in Figure~\ref{fig:model}, is given by \eqref{eq:encoding} and
\begin{align}
\label{eq:decoding}
&\bz = [\bz_1, \bz_2, \cdots, \bz_K] \nonumber \\
&\widehat{\mb{x}} = \mcal{D}(\bz) : \prod_{k=1}^K \mcal{M}_k \to \mcal{X}^n, \\
&\widehat{y}(\widehat{\bx}) = \argmax_{i \in [1:|\mcal{Y}|]} [\mcal{C}(\widehat{\bx})]_i. \nonumber
\end{align}

We refer to a set of encoders $\mathbf{\mathcal{E}}=\{\mcal{E}_k\}_{k=1}^K$ and a decoder $\mcal{D}$ as a {\em distributed quantization system  $(\mathbf{\mathcal{E}},\mcal{D})$.}
Ideally, we would like to use an $(\mathbf{\mathcal{E}},\mcal{D})$ system that minimizes the probability of misclassification. That is, the  encoders and decoder are the solution of the optimization problem
\begin{align}
\label{eq:misclassify_loss}
&\min_{\mathbf{\mathbf{\mathcal{E}}},\mcal{D}:|\mathcal{M}_k|\leq 2^{R_k}}  \mbb{E}_{\bx,y(\bx)\sim p(\bx,y(\bx))}[\mbb{I}(\widehat{y}(\widehat{\mb{x}})\neq y(\bx))] \nonumber \\
&\quad = \min_{\mathbf{\mathbf{\mathcal{E}}},\mcal{D}:|\mathcal{M}_k|\leq 2^{R_k}}  \mbb{E}_{\bx,y(\bx)\sim p(\bx,y(\bx))}[\mbb{I}(\widehat{y}(\mathcal{D}(\mathbf{\mathbf{\mathcal{E}}}(\mathbf{x})))\neq y(\bx))],
\end{align} where: (i) $p(\bx,y(\bx))$ is the input data distribution; (ii) 
 $\widehat{y}(\widehat{\bx})$ and $\widehat{\bx}$ are obtained from $\mb{x}$ using~\eqref{eq:encoding} and~\eqref{eq:decoding}; and
(iii) we used  $\bz = \mathbf{\mathcal{E}}(\bx)$ for brevity.

However, in this paper we assume that the distribution $p(\bx,y(\bx))$ is not known: instead, we are given a dataset $\mcal{T} = \{(\mb{x}^{(i)},y(\bx^{(i)}))\}_{i=1}^N$ which contains $N$ independent samples drawn from  $p(\bx,y(\bx))$. 
Thus, we can only empirically approximate the expectation in~\eqref{eq:misclassify_loss} using the dataset $\mcal{T}$, and hence, our objective is to minimize the {\em misclassification loss} $\mcal{L}(\mathbf{\mathbf{\mathcal{E}}},\mcal{D},\mcal{T})$, calculated as
\begin{align}
\label{eq:misclassify_loss_only}
\mcal{L}(\mathbf{\mathbf{\mathcal{E}}},\mcal{D},\mcal{T}) = \frac{1}{N}\sum_{i=1}^N [\mbb{I}(\widehat{y}(\widehat{\mb{x}}^{(i)})\neq y(\bx^{(i)}))].
\end{align}
In the rest of the paper, we will say that a distributed quantization system
($\mathbf{\mathbf{\mathcal{E}}},\mcal{D}$) is {\em optimal}, if the encoders  $\mathbf{\mathcal{E}} = \{\mcal{E}_k\}_{k=1}^K$ and the decoder $\mcal{D}$ are an optimal solution of the problem
\begin{equation}
\label{eq:misclassify_loss_emp}
\min_{\mathbf{\mathbf{\mathcal{E}}},\mcal{D}:|\mathcal{M}_k|\leq 2^{R_k}}  \mcal{L}(\mathbf{\mathbf{\mathcal{E}}},\mcal{D},\mcal{T}).
\end{equation}

\begin{rem}
{\rm Note that in~\eqref{eq:misclassify_loss_only}, given a labeled
  dataset $\mcal{T},$ our objective is to minimize the empirical
  probability of misclassifying the data points $\bx^{(i)}$ after
  quantization. Instead, if we are given a local dataset of unlabeled
  data points, $\mcal{T}_u= \{\bx^{(i)}\}_{i=1}^N,$ we could create a
  labeled dataset, $\widehat{\mcal{T}} =
  \{(\bx^{(i)},\widehat{y}(\bx)\}_{i=1}^N$, by applying the pretrained
  classifier $\mcal{C}$ on the local unlabeled data $\mcal{T}_u$.  We
  can then apply~\eqref{eq:misclassify_loss_only} - and the logic in the
  remainder of the paper - on the data $\widehat{\mcal{T}} =
  \{(\bx^{(i)},\widehat{y}(\bx)\}_{i=1}^N$.  In this case, our
  objective is equivalent to keeping the classifier output consistent
  before and after applying the distributed quantization system.  }
\end{rem}
	\begin{figure}[h!]
		\centering
		\includegraphics[width=0.9\textwidth]{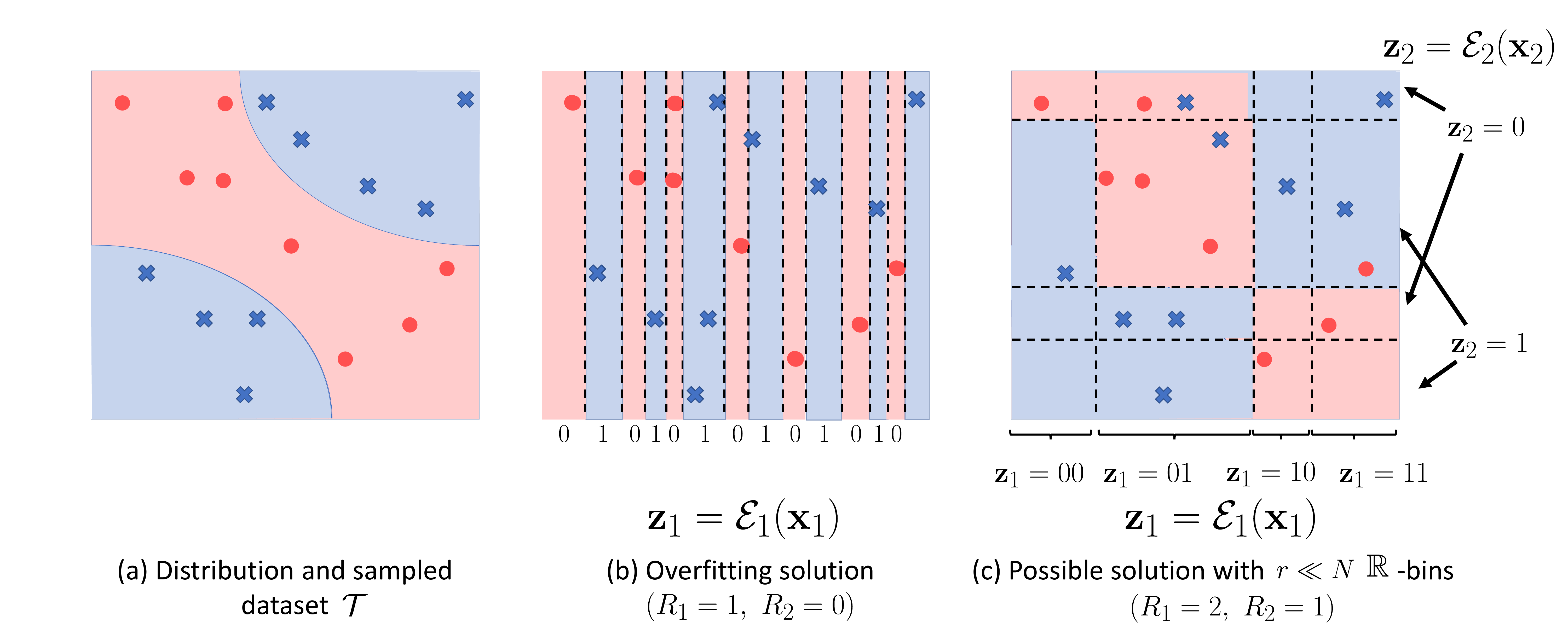}
		\vspace{-1em}
		\caption{Quantization example.}
		\label{fig:preimage_example}
	\end{figure}        
\begin{rem}\label{rem:overfitting}
{\rm 
With no structural restrictions on the encoders $\mathbf{\mathcal{E}}$, it is possible to achieve $\mcal{L}(\mathbf{\mathbf{\mathcal{E}}},\mcal{D},\mcal{T})=0$ almost-surely through over-fitting. For instance, if the distribution $p(\bx)$  is a continuous distribution, the probability that two data points have the same value for $\bx_{\Omega_k}$ is zero for any $k$. Thus, we can consider for example the first node ($k=1$), and partition the space of $\bx_{\Omega_1}$ into $N$ disjoint regions such that each region contains only one training data point. Then, for each region, the encoder function $\mcal{E}_1(.)$ at node $1$ can directly output the class $y(\bx^{(i)})$ of the data point $\bx^{(i)}$ contained in that region, then the decoder outputs a data point that is classified by the classifier to be $y(\bx^{(i)})$.
This requires only $\log_2(|\mcal{Y}|)$ bits. Hence, the rates $[R_1,...,R_K]=[\log_2(|\mcal{Y}|),0,...,0]$ are sufficient to achieve $\mcal{L}(\mathbf{\mathbf{\mathcal{E}}},\mcal{D},\mcal{T}) {=} 0$. 
Although such a quantization scheme will have zero loss when evaluated on the dataset $\mcal{T}$, it would obviously handle out of sample points very poorly.
}
\end{rem}
To avoid this, we restrict the preimage $\mathcal{E}_k^{-1}(\bz_k)$ to be the union of at most $r$ $\mcal{X}_{\Omega_k}$-bins, which are defined below.
\begin{defin}[$\mcal{S}$-bin]
\label{defin:path_connected}
 {\rm
 	We say that the set $\mcal{A} \subseteq \mcal{S}$ is an $\mcal{S}$-bin if $\mcal{A}$ is path-connected~\cite{munkres2014topology} in $\mcal{S}$. A set $\mcal{A}$ is path-connected if and only if for every pair of points $a,b \in \mcal{A}$, there exists a path that connects $a, b$ which completely lies inside $\mathcal{A}$. More formally, a set $\mcal{A}$ is said to be  path-connected if and only if for every pair of points $a,b \in \mcal{A}$, there exists a continuous function $f:[0,1] \to \mcal{A}$ such that $f(0) = a$ and $f(1) = b$~\cite{munkres2014topology}.
 }
 \end{defin} 
By restricting the preimage of $\mathcal{E}_k^{-1}(\bz_k)$ to $r \ll N$ $\mcal{X}_{\Omega_k}$-bins, we force $\mathcal{E}_k$ to assign the same $\bz_k$ to a limited number of path-connected regions (earlier, these could be as many as the number of data points from the same class). By doing so, the solution discussed above (where a single encoder can fully carry the burden of classifying the data points) is eliminated.

We illustrate how the introduction of this restriction can reduce overfitting in the learned quantization system through the example shown in Figure~\ref{fig:preimage_example}.
 Figure~\ref{fig:preimage_example}(a) depicts the underlying true class function $y(\bx)$ through colored regions of $\mbb{R}^2$ and the sampled dataset $\mcal{T}$  as points scattered in the plot. Figure~\ref{fig:preimage_example}(b) shows how an overfitting quantization system (as described in Remark~\ref{rem:overfitting}) using only $\bx_1$ would approximate the underlying class function $y$. Note that, although the resulting system provides poor approximation, it classifies the dataset points perfectly (given by the background color in each region). Finally, Figure~\ref{fig:preimage_example}(c) shows an example where each encoder $\mcal{E}_k$ assigns at most 2 $\mbb{R}$-bins to the same $\bz_k$. It is not difficult to see that although the illustrated quantization system does misclassify some points in the dataset (decision given by the background in each region), it gives a better approximation of $y(\bx)$ outside the given dataset compared to the design in~\ref{fig:preimage_example}(b).

\medskip 

A summary of the notation used throughout the paper is given  in Table~\ref{tbl:notation}.
	\begin{table}
		\centering
		\begin{tabular}{@{}cl||cl@{}}
			\toprule
			{\textbf{Symbol}} & {\bf Description} & {\textbf{Symbol}} & {\bf Description}  \\ 
			\bottomrule
			$K$						& Number of distributed sensing nodes & 
			$\mcal{M}_k$			& Set of possible values for $\bz_k$ ({\bf def.}~\eqref{eq:encoding})
			\\
			$n$             		& Number of features of data point $\bx$&
			$\bz$	   				& Collection of encoder outputs $[\bz_1,\bz_2,\cdots,\bz_K]$ ({\bf def.}~\eqref{eq:decoding})
			\\
			$\mcal{Y}$      		& Set of possible classes & 
			$\mcal{M}$			    & Set of possible values for $\bz$ ({\bf def.}~\eqref{eq:decoding})
			\\
			$y$						& True class label of data point $\bx$ &
			$\widehat{\bx}$ 		& Reconstructed input from $\bz$ using $\mcal{D}$ ({\bf def.}~\eqref{eq:decoding})
			\\			
			$\Omega_k$      		& Set of features at node $k$ &
			$\mcal{L}(\mcal{E},\mcal{D},\mcal{T})$ & Empirical misclassification loss ({\bf def.}~\eqref{eq:misclassify_loss_only}) 
			\\
			$R_k$					& Number of bits/data point at node $k$ &
			$\mcal{S}$-bin			& Path-connected subset of $\mcal{S}$ ({\bf def.} Definition~\ref{defin:path_connected})
			\\
			$\mcal{C}$ 				& Pretrained classifier ({\bf def.}~\eqref{def:C}) &
			$p_e$					& Misclassification loss threshold ({\bf def.} Lemma~\ref{lem:poly_equiv})
			\\
			$\widehat{y}$			& Output class by classifier $\mcal{C}$ ({\bf def.}~\eqref{def:y_hat}) &
			$d$						& Set of boundaries used in optimal on-the-line quantizer 
			\\
			$\mcal{E}_k$ 			& Encoder at node $k$ ({\bf def.}~\eqref{eq:encoding}) &
			$f_k(\cdot;\theta_k)$   & Neural network of neural driven encoder $\mcal{E}_k$
			\\
			$\mcal{E}_k^{-1}$	 	& Preimage of $\mcal{E}_k$({\bf def.}~\eqref{eq:preimage}) &
			$\mbb{R}^{m_k}$  		& Output space of neural encoder $f_k$ ({\bf def.}~\eqref{eq:neural_encoding})
			\\
			$\mcal{E}$ 				& Collection of encoders $\{\mcal{E}_k\}_{k=1}^K$ &
			$\mcal{Q}_k$			& Quantizer of neural driven encoder $\mcal{E}_k$ ({\bf def~\eqref{eq:neural_encoding}})
			\\
			$\mcal{D}$ 				& Decoder at central node ({\bf def.}~\eqref{eq:decoding}) & 
			$\mcal{G}$				& Initial mapping of neural driven decoder $\mcal{D}$ ({\bf def.}~\eqref{eq:neural_decoding}) 
			\\
			$\mcal{T}$ 				& Training dataset $\{(\bx^{(i)},y(\bx^{(i)})\}_{i=1}^N$ &
			$g(\cdot;\phi)$   & Neural network of neural driven decoder $\mcal{D}$ ({\bf def.}~\eqref{eq:neural_decoding})
			\\
			$N$						& Number of data points in $\mcal{T}$ & 
			$\mcal{L}^{\rm c}$ & Misclassification loss for neural based approach ({\bf def.}~\eqref{eq:misclassify_loss_emp3}) 			
			\\			
			$\bz_k$	   				& Output of encoder $\mcal{E}_k$ ({\bf def.}~\eqref{eq:encoding}) &
			$\mcal{L}^{\rm q}$ & Quantization loss for neural based approach ({\bf def.}~\eqref{eq:quantization_loss})
			\\
			\bottomrule
		\end{tabular}
		\caption{Notation used throughout the paper.}
		\label{tbl:notation}
	\end{table}

	\section{On the complexity of finding an optimal distributed quantization system}
\label{sec:complexity}
In this section we study the complexity of finding an optimal quantization system $(\mathbf{\mathcal{E}},\mcal{D})$ that minimizes the loss in~\eqref{eq:misclassify_loss_emp} over the dataset $\mcal{T}$. 
We first start by describing how to find an optimal decoder $\mathcal{D}$ assuming that the optimal encoders $\mathbf{\mathcal{E}}^\star = \{\mathcal{E}^\star_k\}_{k=1}^K$ are given. We then discuss the complexity of finding optimal encoders and show that the problem is NP-hard in all cases but one. 
For the case where the problem is not NP-hard, we propose a polynomial-time algorithm to find the optimal quantization system (encoders/decoder) under some structural restrictions on the encoders.

\subsection{Optimal decoder} \label{optimal_dec}
Assuming that the optimal encoders $\mathbf{\mathcal{E}}^\star$ are given, we are interested in a decoder $\mathcal{D}^\star$ that minimizes the misclassification loss in~\eqref{eq:misclassify_loss_emp}.
For brevity, let us denote the set of all possible encoded values $\mathbf{z} = \mathbf{\mathcal{E}^\star}(\bx)$ as $\mcal{M}$, i.e., 
\begin{align}\label{eq:z_in_M}
	\mathbf{z}\in \mathcal{M},\quad \text{where}\quad \mathcal{M}=\prod_{k=1}^K \mcal{M}_k.
\end{align}
The operation of the optimal decoder is described in the following lemma.
\begin{lem}\label{lem:opt_decoder}
	For given fixed encoders $\mcal{E}^\star$, the optimal decoder $\mcal{D}^\star$ is defined by
	\begin{equation}
		\mathcal{D}^\star(\mathbf{z}) = \widehat{\bx}\quad \text{s.t.}\quad	\widehat{y}(\widehat{\bx}) = \argmax_{c \in \mcal{Y}} \sum_{i:\mathcal{E}^\star(\bx^{(i)}) = \bz}\mbb{I}\left[y(\mb{x}^{(i)}) = c\right],
	\end{equation}
	where $\widehat{y}$, defined in~\eqref{def:y_hat}, is the label output by the classifier $\mcal{C}$ for $\widehat{\bx}$.
\end{lem}
\begin{proof}
	The misclassification loss in~\eqref{eq:misclassify_loss_only} can be rewritten as 
	\begin{equation} \label{obj_1}
	\mcal{L}(\mathbf{\mathbf{\mathcal{E}^\star}},\mcal{D},\mcal{T})=\frac{1}{N}\sum_{\mathbf{z}\in \mcal{M}} \sum_{i:\mathbf{\mathbf{\mathcal{E}^\star}}(\mathbf{x}^{(i)})=\mathbf{z}}\mbb{I}\left[y(\mb{x}^{(i)})\neq \widehat{y}(\mcal{D}(\mb{z}))\right],
	\end{equation} 
	where $\widehat{y}(\mcal{D}(\mb{z}))$ is obtained by \eqref{eq:decoding}. 
	Since for a fixed $\bz$, $\mcal{D}(\mb{z})$ only affects one term in the outer summation in~\eqref{obj_1}, each of the outer summation terms can be independently minimized by choosing $\mathcal{D}^\star(\mathbf{z})$ to be a point $\widehat{\bx} \in \mcal{X}^n$ satisfying
	\begin{align}
	\widehat{y}(\widehat{\bx}) &= \argmin_{c \in \mcal{Y}} \sum_{i:\mathbf{\mathbf{\mathcal{E}^\star}}(\mathbf{x}^{(i)})=\mathbf{z}}\mbb{I}\left[y(\mb{x}^{(i)}) \neq c\right] = \argmin_{c \in \mcal{Y}} \left(N-\sum_{i:\mathbf{\mathbf{\mathcal{E}^\star}}(\mathbf{x}^{(i)})=\mathbf{z}}\mbb{I}\left[y(\mb{x}^{(i)}) = c\right]\right)\nonumber \\
	&=  \argmax_{c \in \mcal{Y}} \sum_{i:\mathbf{\mathbf{\mathcal{E}^\star}}(\mathbf{x}^{(i)})=\mathbf{z}}\mbb{I}\left[y(\mb{x}^{(i)}) = c\right].
	\end{align}
\end{proof}
That is, $\widehat{\bx} = \mathcal{D}^\star(\mathbf{z})$ can be any point in $\mathcal{X}^n$ such that $\widehat{y}(\widehat{\bx})$ (the decision of classifier $\mcal{C}$ for $\widehat{\bx}$) is the majority true label $y(\bx)$ among the points of the dataset $\mcal{T}$ that fall in $\mathbf{\mathcal{E}}^{-1}(\mb{z})$.
For instance, in the example shown in Figure~\ref{fig:preimage_example}, if $\mathbf{\mathcal{E}}^{\star-1}(\mb{z})$ has one ``x'' (blue)  and two ``o'' (red) training points, any point that the classifier $\mcal{C}$ would classify to be ``o'' (red) can be selected as $\widehat{\mb{x}} = \mcal{D}^\star(\bz)$.
Thus, the optimal decoder $\mcal{D}^\star$ manipulates the classifier $\mcal{C}$ to output a classification that best serves the loss function in~\eqref{obj_1}.

With the optimal decoder in mind, we are now ready to discuss the hardness of the problem of finding the optimal encoders $\mathbf{\mathcal{E}}$ in the following subsection.

\subsection{Hardness of finding an optimal quantizer}
\label{subsec:hardness}
Given a training dataset $\mcal{T} = \{(\bx^{(i)},y^{(i)})\}_{i=1}^N$, our goal is to design an \emph{optimal} distributed quantization system $(\mathbf{\mathcal{E}},\mcal{D})$ which minimizes the misclassification loss in~\eqref{eq:misclassify_loss_only} for a given communication budget of $R_k$ bits per data point at each node $k$.
We study four different cases of the problem:
\begin{enumerate}[{\bf (P1)}]
	\item For number of features $n > 1$, number of classes $|\mcal{Y}| > 1$, dataset $\mcal{T} = \{(\bx^{(i)}, y^{(i)})\}_{i=1}^N$ and given $\{R_k\}_{k=1}^K$: find the optimal ($\mathbf{\mathcal{E}}$, $\mcal{D}$) that minimizes the misclassification loss $\mcal{L}(\mathbf{\mathbf{\mathcal{E}}},\mcal{D},\mcal{T})$, assuming that $\mcal{E}_k^{-1}(\bz_k)$ is the union of $r < N$ $\mathcal{X}_{\Omega_k}$-bins. \label{proba}
	\item The restriction of problem \ref{proba} to the case of linearly separable data. \label{probb}
	\item The restriction of problem \ref{proba} to the case where $\mcal{E}_k^{-1}(\bz_k)$ is a single $\mathcal{X}_{\Omega_k}$-bin. \label{probc}
	\item The restriction of problem \ref{probc} to the case of linearly separable data. \label{probd}
\end{enumerate}

Next, we prove that the first three problems for $n>1$ and the last problem for $n>2$ are NP-hard and also prove their hardness of approximation. 
In the general case, as a result of the hardness and hardness of approximation, we focus on finding heuristic approaches to find a good distributed quantization solution that may not necessarily be optimal.

\begin{rem}
{\rm 
For problem~\ref{probd}, we prove the hardness results for number of features $n \geq 3$. In the next subsection, we introduce optimal polynomial-time algorithm for the case $n = 2$ under some structural restrictions on the encoders. 
}
\end{rem}

For all the problems, in order to prove the goal results, it is sufficient to consider prototype settings with predefined number of features $n$, number of classes $|\mcal{Y}|$, number of nodes $K$, and communication budget $\{R_k\}_{k=1}^K$ and allow the size $N$ of the dataset $\mcal{T}$ to grow. It follows that the general problems, which are expansions of these prototype problems, are also NP-hard. In particular, in all cases, we assume that each distributed node quantizes only one feature (i.e., $n=K$) and the number of classes $|\mcal{Y}| = 2$. The remaining parameters are defined below for each problem
\begin{align}
	\label{eq:prob_parameters}
	&\bullet \ref{proba}\ \&\ \ref{probb}: n = K = 2,\ \text{finite}\ R_1\ \text{and}\ R_2 \to \infty;\nonumber \\
	&\bullet \ref{probc}: n = K = 2,\ \text{finite}\ R_1 = R_2;\\
	&\bullet \ref{probd}: n = K = 3,\ R_3 = 0\ \text{and finite}\ R_1 = R_2.\nonumber
\end{align}

We start by showing that, under polynomial-time reductions, the problem of finding the optimal quantization system in~\ref{proba}-~\ref{probd} with the aforementioned parameters is equivalent to finding the minimum number of bits $\{R_k\}_{k=1}^K$ required for a particular fixed misclassifcation error $p_e$. In particular, the equivalence is summarized in the following lemma.

\begin{lem}	\label{lem:poly_equiv}
	For a fixed $p_e \in [0,1]$ and number of classes $|\mcal{Y}| = 2$: the problems \ref{proba} - \ref{probd} with parameters in \eqref{eq:prob_parameters} are equivalent to \ref{i} - \ref{iv} below under polynomial-time reductions:
	\begin{enumerate}[{\bf(P1')}]
		\item For number of features $n = 2$, $R_2\to \infty$, finding the minimum $R_1$ for which $\mcal{L}(\mathbf{\mathbf{\mathcal{E}}},\mcal{D},\mcal{T})<p_e$, assuming that $\mcal{E}_k^{-1}(\bz_k)$ is the union of $r < N$ $\mathcal{X}_{\Omega_k}$-bins. \label{i}
		\item The restriction of problem \ref{i} to the case of linearly separable data. \label{ii}
		\item For number of features $n = 2$, finding the minimum $R_1=R_2$ for which $\mcal{L}(\mathbf{\mathbf{\mathcal{E}}},\mcal{D},\mcal{T})<p_e$, assuming that $\mcal{E}_k^{-1}(\bz_k)$ is a single $\mathcal{X}_{\Omega_k}$-bin. \label{iii}
		\item For number of features $n {=} 3$, $R_3{=}0$, finding the minimum $R_1=R_2$ for which $\mcal{L}(\mathbf{\mathbf{\mathcal{E}}},\mcal{D},\mcal{T})<p_e$ for the case of linearly separable data, assuming that $\mcal{E}_k^{-1}(\bz_k)$ is a single $\mathcal{X}_{\Omega_k}$-bin. \label{iv}
	\end{enumerate}
\end{lem}	
\begin{proof}
	 The proof is based on the observation that the loss $\mcal{L}(\mathbf{\mathbf{\mathcal{E}}},\mcal{D},\mcal{T})$ can only take one of the $N+1$ values: $0,\frac{1}{N},\frac{2}{N},...,\frac{N}{N}$, and $2^{R_1}$ can only take values in $[1:N]$. Hence, if we have a polynomial-time algorithm which solves problem \ref{proba} in $O(f(N))$, then we can answer the mentioned question in $O(Nf(N))$ by finding the minimum loss $\forall R_1 \in [\log_2(1), \log(2), \cdots, \log_2(N)]$ and pick the smallest $R_1$ for which the minimum achieved loss is less than $p_e$. Similarly, if we have a polynomial-time algorithm that answers this question, then we can solve problem \ref{proba} in polynomial time. Hence, \ref{i} and problem \ref{proba} are equivalent under polynomial-time reduction. 
	 Following the same logic, problem \ref{probb} - \ref{probd} are equivalent to~\ref{ii} - \ref{iv}.
\end{proof}
Based on Lemma~\ref{lem:poly_equiv}, the hardness results can now be proved by working directly with~\ref{i} - \ref{iv}.
In particular, these problems are NP-hard as stated in the following theorem.

\begin{theorem} \label{hardness}
	For a fixed $p_e \in [0,1]$ and number of classes $|\mcal{Y}| = 2$, the problems \ref{i} - \ref{iv} are NP-hard. Moreover, we have that
		\begin{itemize}
		\item Approximating $2^{R_k}$ in problems \ref{i}, \ref{ii} within $O(N^{1-\epsilon})$ is NP-hard $\forall \epsilon>0$.
		\item Approximating $2^{R_k}$ in problems \ref{iii}, \ref{iv} within $O(N^{\frac{1}{2}-\epsilon})$ is NP-hard $\forall \epsilon>0$ assuming the \textit{Small Set Expansion Hypothesis} (SSEH) and that NP $\nsubseteq$ BPP.
	\end{itemize}
\end{theorem}
\begin{proof}
	We prove the results for each problem by reduction from an NP-Complete problem. In particular, we prove the result for \ref{i}, \ref{ii} by reduction from the vertex coloring problem, and for \ref{iii}, \ref{iv} by reduction from the maximum balanced biclique problem. The proof is delegated to Appendix \ref{hardness_proof}.
\end{proof}

\subsection{Optimal Quantizer for Linear Classifiers in 2D}
\label{sec:opt_line}
In the previous section, we proved that, for problem~\ref{probd} when $n>3$, it is NP-hard to find an optimal quantization system. 
In this subsection, we propose an optimal polynomial-time algorithm for problem~\ref{probd} when the number of features is $n=2$ under some structural restrictions on the encoders. Specifically, we consider a system with two distributed nodes.  Each node $k\in [1:2]$,  observes one feature $x_k$, and aims to quantize $x_k$ using $R$ bits. We assume that we have two classes ($|\mcal{Y}|=2)$ to distinguish among and that the data is linearly separable, namely, $\mcal{C}$ is a linear classifier with output $ \widehat{y}(\mb{x})=y(\mb{x})$. Moreover, without loss of generality, we assume that the features are scaled and translated such that the line  $x_1 =x_2$ separates the data. This transformation can be performed during encoding at each distributed node and reverted in the decoder $\mcal{D}$ at the central node.
\begin{figure}[!h]
	\centering
	\includegraphics[width=0.4\textwidth]{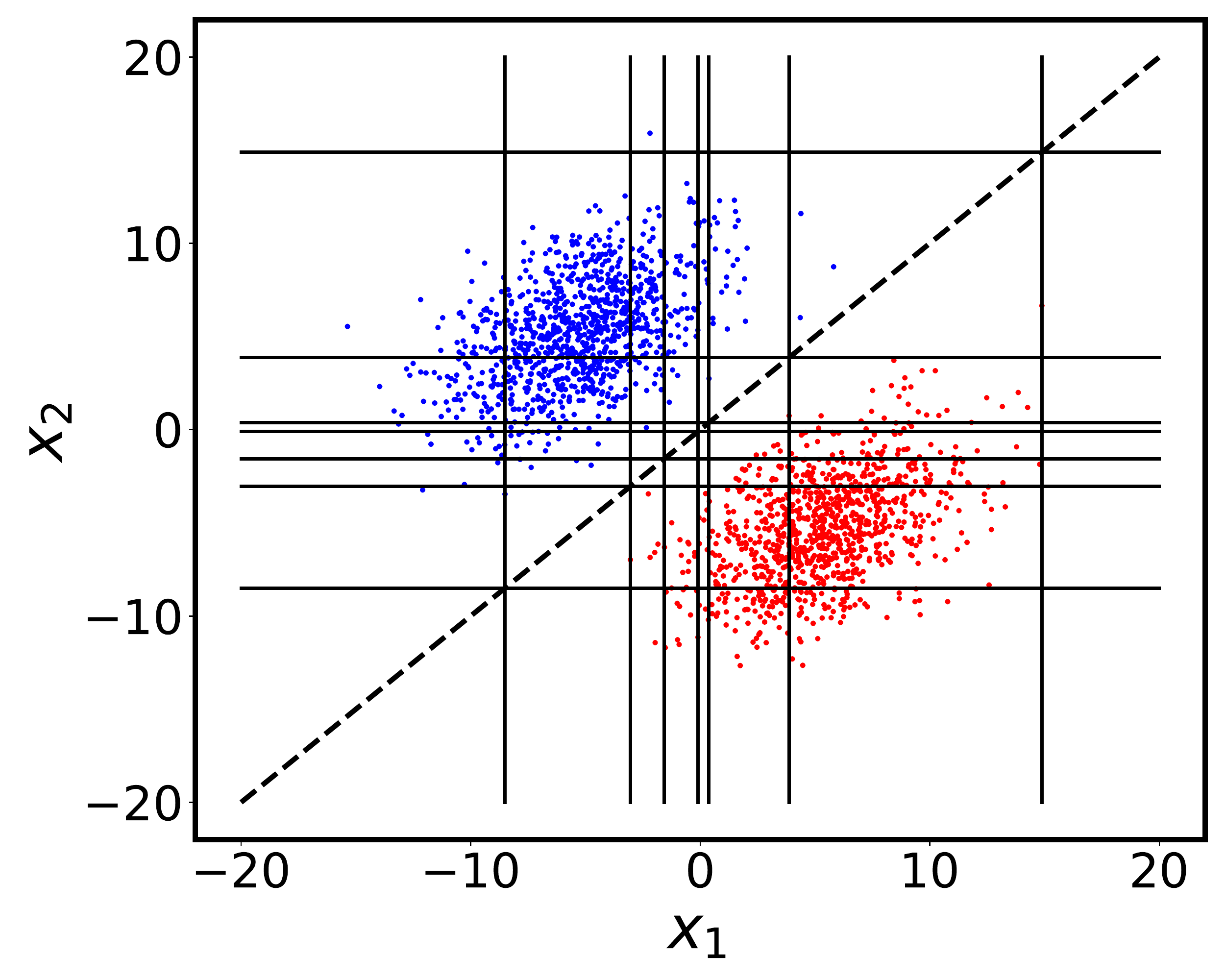}
	\caption{Example of \texttt{on-the-line} quantizer, where boundaries for $x_1$ and $x_2$ intersect along the $x_1 = x_2$ line ($45^\circ$ line).}
	\label{fig:on_line_example}
\end{figure}
Note that in this case, since $\mathcal{X}_{\Omega_k}=\mathbb{R}, \forall k \in[1:2]$ and an $\mathbb{R}$-bin is an interval $[a,b]$ for some $a,b$, then the encoder/quantizer at node $k$ divides $\mathcal{X}_{\Omega_k}=\mathbb{R}$ into $2^R$ intervals by introducing the quantization boundaries $(d_{k,1}, \cdots, d_{k,2^{R}-1})$. We here further restrict our attention to the class of {\bf on-the-line} quantizers, where the horizontal and vertical lines defining $d_{k,i}$ meet along the line $x_1=x_2$ (as in Figure~\ref{fig:on_line_example}).  
In other words, nodes $1$ and $2$ use a common encoder design $\mcal{E}_1 = \mcal{E}_2$. 
This implies that $d_{0,k}=d_{1,k}=d_{k}$, and thus we simply need to find the $2^R-1$ quantization boundaries $(d_1, \cdots, d_{2^R-1})$.
\begin{rem}\label{rem:2N}
	{\rm 
		Note that although $d_{k}$ can take any value in $\mbb{R}$, only $2N$ values can make a difference in the misclassification loss in \eqref{eq:misclassify_loss_only}:  the $2N$ values corresponding to either coordinate of the training data points $\{\mb{x}|(\mb{x},y(\mb{x})) \in \mcal{T}\}$. Indeed, these are the only boundaries that can change the bin to which a training point belongs\footnote{If a point lies on a boundary, we assume it belongs to the bin preceding that boundary.}.
	}
\end{rem}
To find the on-the-line optimal quantizer, we could simply do an exhaustive search over all possible $2N$ values (recall Remark~\ref{rem:2N}) that each of the boundaries $d_{k}$ can take which costs a complexity of $O\left( \binom{2N}{2^R}\right)$, which is not efficient. Instead, we use a recursive approach, that is based on the following two key observations.
\begin{rem}\label{rem:diagonal_bins}
	{\rm 
		The encoders/quantizers $\mcal{E}_1, \mcal{E}_2$ decompose $\mbb{R}^2$ into $2^{2R}$ $\mbb{R}^2$-bins. However, since the data points are assumed to be linearly separable by the line $x_1=x_2$, then we only need to consider the $\mbb{R}^2$-bins crossed by $x_1=x_2$ as the sources of misclassification. In particular, any other $\mbb{R}^2$-bin is completely populated by data points from the same class.
	}
\end{rem}
\begin{rem}\label{rem:induction_obs}
	{\rm
		Assume that a vector $\mathbf{s}$ lists the $2N$ possible boundary values in ascending order, i.e., $s_j \leq s_i$ for all $j \leq i \leq 2N$.
		Let $\mcal{T}_{s_i}$ denote the subset of the dataset $\mcal{T}$ such that both coordinates of $\bx$ are upper bounded by $s_i$, i.e., 
		\begin{align}
		\mcal{T}_{s_i}  = \{(\mb{x},y(\mb{x})) \in \mcal{T} | x_1, x_2 \leq s_i \},\ \forall i \in [1:2N].
		\end{align}
		Note that $\mcal{T}_{s_j} {\subseteq} \mcal{T}_{s_i}, \forall i {<} j$ and that $\mcal{T}_{s_{2N}} {=} \mcal{T}$.
		Then, the optimal quantizer with $b$ boundaries on $\mcal{T}_{s_i}$ shares $b-1$ boundaries with the optimal quantizer with $b-1$ boundaries on $\mcal{T}_{s_j}$ for some $j < i$.
	}
\end{rem}
Remark~\ref{rem:induction_obs} is restated and proved in Appendix~\ref{proof:on-the-line}.
Observations in Remark~\ref{rem:diagonal_bins} and Remark~\ref{rem:induction_obs} lead to a polynomial-time dynamic-programming algorithm to design the optimal quantization boundaries $(d_1,\cdots,d_{2^R -1})$.

The algorithm's pseudo code is given in Algorithm~\ref{alg3} and implements the following logic:
Given an ordered list of the possible $2N$ boundary values $\mb{s}$, 
let $E(s_i,b)$ be the minimum number of misclassified points over the dataset subset $\mcal{T}_{s_i}$ when using $b$ boundaries and $A(s_i,b)$ be the set of boundaries that achieve this minimum loss. 
Then in each iteration $i \in [1:2N]$:

\begin{enumerate}
	\item Find $E(s_i,b), \forall b \in [1:2^R-1]$ by trying to augment $A(s_j,b-1), \forall j < i$ with one extra boundary at $s_j$ (Remark~\ref{rem:induction_obs}). The additional number of misclassified points is only a result of the points in $\{(\bx,y(\bx))\in \mcal{T}| s_j < x_1 \leq s_i, s_j < x_2 \leq s_i \}$ (Remark~\ref{rem:diagonal_bins});
	\item Retain the best augmentation $A(s_i,b)$ to be used in the following iteration;
	\item After $2N$ iteration, the optimal quantization boundaries are stored in $A(s_{2N},2^R-1)$.
\end{enumerate}
In the worst case the algorithm does $2N \times 2^R$ iterations over the whole dataset $\mcal{T}$ resulting in a time-complexity of $O(N^2 2^R)$.
The optimality of Algorithm~\ref{alg3} is proved in Appendix \ref{proof:on-the-line} as a consequence of proving the observation in Remark~\ref{rem:induction_obs}.
\begin{algorithm}
	\caption{Optimal on-the-line quantizer for linearly separable data in $\mbb{R}^2$ }
	\label{alg3}  
	\begin{algorithmic}
		\State {\bf Input:} (a) Training set $\{(\mb{x}^{(i)},,y(\bx^{(i)}))\}_{i=1}^N$; (b) Quantization bits/feature $R$; (c) Ordered set $\mb{s}$ of potential boundary values.\\
		By $\mb{x} \preceq p$ we express that all elements in $\mb{x}$ are less than $p$; $p \prec \mb{x}$ means they all exceed $p$
		\State {\bf Output:} Quantization boundaries $(d_1, d_2, \cdots, d_{2^R-1})$ to use for features $x_1$ and $x_2$
		\State {\bf Initialize:}\\
		\hspace*{\algorithmicindent} $E(s,0) \gets \displaystyle\min_{c \in \{1,2\}} |\{j | \mb{x}^{(j)} \preceq s,\ y^{(j)} = c\}|\quad $ for $s \in \mb{s} $
		\For{$i \in [1:|\mb{s}|]$}
		\For {$b \in [1:2^R-1]$}
		\State $E(s_i,b) \gets \displaystyle\min_{\ell < i}\left\{E(s_\ell,b-1) + \displaystyle\min_{c \in \{1,2\}} |\{j | s_\ell \prec \mb{x}^{(j)} \preceq s_i,\ y^{(j)} = c\}| \right\}$
		\State $\ell^\star \gets$ index $\ell$ that gave the minimum value for $E(s_i,b)$ in the previous expression
		\State $A(s_i,b) \gets A(s_{\ell^\star},b-1) \cup \{s_{\ell^\star}\}$
		\EndFor
		\EndFor
		\State {\bf return }$A(s_{|\mb{s}|},2^R-1)$
	\end{algorithmic}
\end{algorithm}

In the following section, we introduce an approach for designing the encoders and decoder in a more general setting, i.e., when the data points are not necessarily linearly separable and the number of nodes and features are greater than or equal to $2$.

\section{Greedy Boundary Insertion ($\texttt{GBI}$) quantizer}
\label{sec:GBI}
We refer the reader to Table~\ref{tbl:notation}  for the system notation used in this section.
Here, we propose our Greedy Boundary Insertion ($\texttt{GBI}$) algorithm to design encoders/quantizers $\mcal{E}_k, \forall k \in [1:K]$, that can be executed in polynomial-time in the dataset size $N$ and the number of features $n$ for any number of classes.
For the decoder $\mcal{D}$, we use the optimal decoder derived in Section~\ref{optimal_dec}.
$\texttt{GBI}$ extends the intuition in the observations in Subsection \ref{sec:opt_line} to a more general case, where the classifier is arbitrary,
and where each distributed node $k$ observes $\Omega_k$ features and can have arbitrary rate $R_k$. 
We design encoders/quantizers such that $\mcal{E}_k^{-1}(\bz_k)$ (see~\eqref{eq:preimage}) is a single $\mathcal{X}_{\Omega_k}$-bin.
Note that, since we are not constrained to use the same boundaries for each feature as in the on-the-line case, it is sufficient to consider $N$ possible boundary values per feature, the values taken at that feature by the $N$ training points.

The logic behind
$\texttt{GBI}$ is as follows.
$\texttt{GBI}$ iteratively adds quantization boundaries selected greedily: at each iteration it selects to add one of the possible $N$ boundaries to one of the $n$ features, the one that minimizes the misclassification loss in~\eqref{eq:misclassify_loss_emp} given the choice of boundaries in the previous iterations\footnote{Since GBI adds a boundary for one feature at a time, instead of a function of the features $\bx_{\Omega_k}$, we end up with an encoder of the form of a rectangular-grid, where each region is assigned to a value of $\bz_k$.}.
A feature $i$ can accept a new  boundary, if $i \in \Omega_k$ for some node $k$ and introducing a new boundary for feature $i$ does not cause node $k$ to have more than $2^{R_k}$ $\mbb{R}^{|\Omega_k|}$-bins. 
The algorithm terminates when none of the features can accept a new boundary. 
If two or more possible boundaries lead to the same loss (something that happened surprisingly often in our experiments), then instead of breaking ties at random, it makes a significant performance difference to break ties by using a non-linear criterion. This criterion penalizes a boundary that leaves $\mbb{R}^n$-bins with high individual misclassification to correct 
classification ratio. This is discussed in more detail in Appendix~\ref{sec:purity_criterion}. 

The pseudo code for $\tt GBI$ is presented in Algorithm~\ref{alg3_greedy}. The losses computed in Algorithm~\ref{alg3_greedy} assume that the optimal decoder (defined by Lemma~\ref{lem:opt_decoder}) for the designed encoders is used.


Using this notation, the pseudocode for $\texttt{GBI}$ is presented in Algorithm~\ref{alg3_greedy}.

\begin{algorithm}
    \caption{\texttt{GBI} Algorithm}
    \label{alg3_greedy} 
    \begin{algorithmic}
        \State {\bf Variables: }
        \begin{align*}
            \{{\bf d}\}_{f=1}^n &: \text{Boundaries for feature $f \in [1:n]$;} \\
            B_k &: \text{Number of $\mbb{R}^{\Omega_k}$-bins used by node $k \in [1{:}K]$ using boundaries $\{\mb{d}_f\}_{f \in \Omega_k}$;}\\
            \mcal{L}(\mb{d}_1,...,\mb{d}_n) &: \text{misclassification loss (defined in~\eqref{eq:misclassify_loss}) using encoders defined by $\{{\bf d}\}_{f=1}^n$};\\
            \Delta_k(f, \{\mb{d}_j\}_{j \in \Omega_k}) &: \text{Increase in $B_k$, if a new boundary is introduced for feature $f$}.
        \end{align*}
        \State {\bf Input:} (a) Training set $\{(\mb{x}^{(i)},y(\bx^{(i)}))\}_{i=1}^N$; 
        (b) Quantization bits/node $R_k, \forall k \in [1:K]$.
        \State {\bf Output:} Quantization boundaries $\mb{d}_f=\{d_{f,1}, d_{f,2}, \cdots\}$, $\forall f \in [1:n]$
        \State {\bf Initialize:}\\
        \hspace*{\algorithmicindent} 1) List $\mb{s}_f$ of potential boundaries for feature $f$ from the training set, i.e.,
        $\mb{s}_f \gets \{x^{(i)}_f\}_{i=1}^N $
        \\
        \hspace*{\algorithmicindent} 2) Number of current bins in node $k$,  $B_k=0$, boundaries $\mb{d}_f=\phi$ for each feature $f$
        \While{$\exists k: B_k+\min_{f\in \Omega_k}\Delta_k(f, \{\mb{d}_j\}_{j \in \Omega_k})\leq 2^{R_k}$}
        \State - Among all $k$ satisying condition above, find $\widehat{f} \in \Omega_k,\ {\widehat{d}} \in \mb{s}_{\widehat{f}}$ that minimizes \[\mcal{L}(\mb{d}_1,...,\mb{d}_{\widehat{f}}\cup {\widehat{d}},...,\mb{d}_n)\]
        \State - If there are more than one pair $(\widehat{d}, \widehat{f})$ achieving the same minimum objective value,\State we break ties using a non-linear criterion (Appendix~\ref{sec:purity_criterion})
        \State {\bf Update:}
        \begin{align*}
            & \mb{d}_{\widehat{f}}\gets \mb{d}_{\widehat{f}}\cup\{\widehat{d}\} \\
            & B_k \gets B_k + \Delta_k(\widehat{f}, \{\mb{d}_j\}_{j \in \Omega_k}),\quad \text{where $\widehat{f} \in \Omega_k$} 
        \end{align*}
        \EndWhile
        
    \end{algorithmic}
\end{algorithm}

\medskip

\par{\bf Complexity of $\texttt{GBI}$.}\quad At each iteration of the algorithm, we compute the reduction in misclassification error associated with every potential boundary and pick the boundary with the most reduction. This involves $O(N)$ operations per boundary. Thus to add a single boundary, $O(nN^2)$ operations are needed in the worst-case. This results in time-complexity of $O(nN^2 2^{R_{\max}})$, where $R_{\max} = \max_k R_k$. Recall that our focus is on cases where the number of bits used are much lower than required for full precision ($32$ bits). As a result, the contribution of $R_k$ in the complexity term can be subsumed into the notation $O(nN^2)$.

\medskip

\begin{rem}
{\rm 
Despite the fact that $\texttt{GBI}$ is a polynomial-time algorithm, we are interested in approaches with linear complexity in $N$, as the number of available data points (as well as features) in a dataset can be large. 
To overcome the effect of quadratic complexity in $N$, ${\tt GBI}$ can be applied stochastically by randomly sampling a subset of the dataset $\mcal{T}$ to use at each iteration (instead of evaluating the decrease in misclassification over the whole dataset $\mcal{T}$). 
}
\end{rem}

\begin{rem}
{\rm 
A possible drawback of ${\tt GBI}$ is that boundaries are directly introduced on the native features without transformation. 
Thus, as aforementioned, the resulting encoder $\mcal{E}_k$ at node $k$ would always have a rectangular grid structure where each area in the grid would be assigned to some $\bz_k$.
It is not difficult to see that allowing a transformation on the features available at node $k$ (i.e., $\Omega_k$), can allow more elaborate encoder designs. We study how to design such a transformation before applying $\tt GBI$ as part of the deep learning approach proposed in the following section.
}
\end{rem}

\section{Distributed quantization for classification tasks using neural representations}
\label{sec:neural}
In this section, we explore a learning based approach for the distributed quantization problem introduced in Section~\ref{sec:model}. We consider a quantization system where the encoders $\{\mcal{E}_k\}_{k=1}^K$ and decoder $\mcal{D}$ are neural networks, followed by a pretrained classifier $\mcal{C}$ that is subdifferentiable. 

\begin{figure}[h!]
	\centering
	\includegraphics[width=0.9\textwidth]{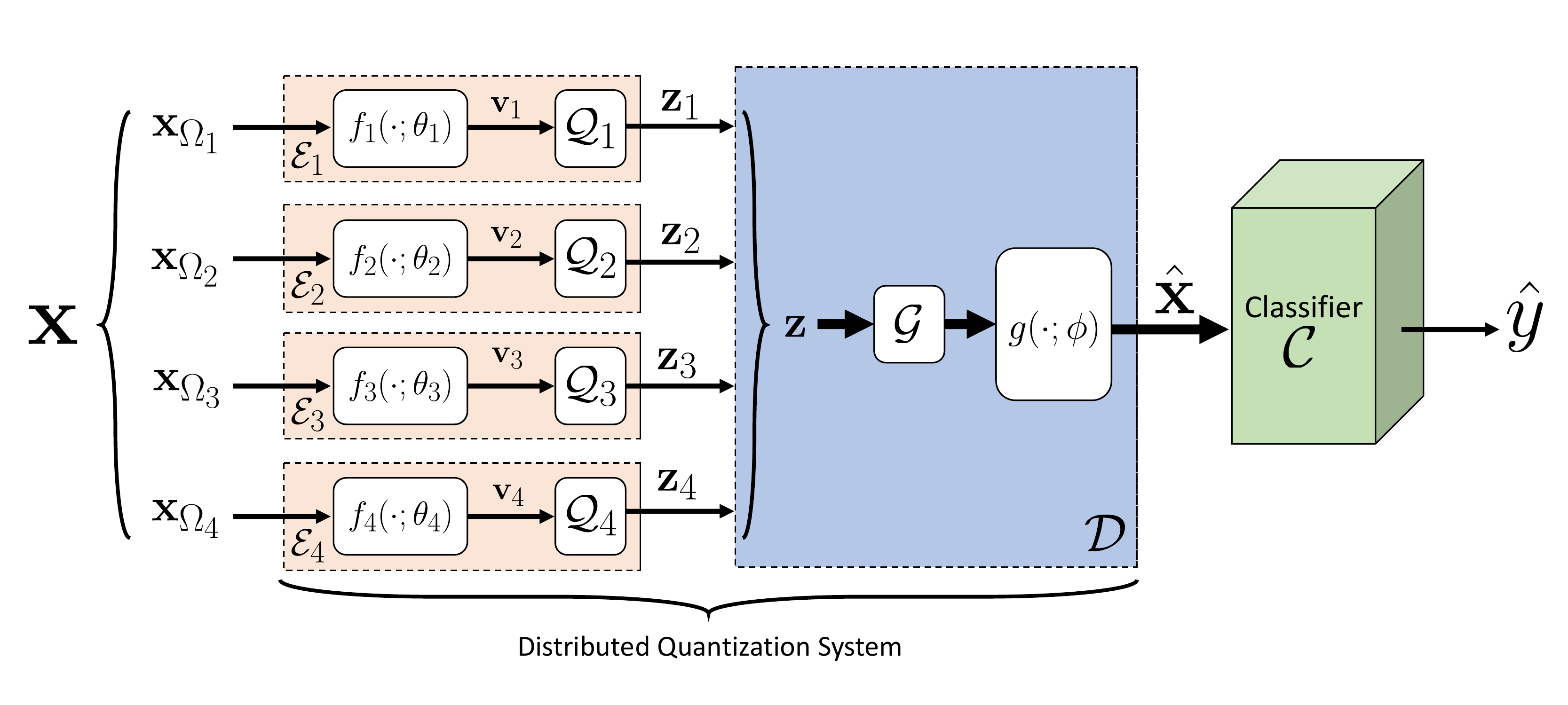}
	\caption{An example of the different components of encoders and decoders in the distributed quantization system for $K=4$ nodes.}
	\label{fig:model_nn}
\end{figure}
The structure of the encoders and decoder is shown in Figure~\ref{fig:model_nn}. 
In particular, the encoder $\mcal{E}_k(\cdot)$ is decomposed into a neural network parameterized by $\theta_k$, which implements a function $f_k(\cdot;\theta_k): \mathcal{X}_{\Omega_k} \to \mbb{R}^{m_k}$, followed by a quantizer $\mcal{Q}_k:\mathbb{R}^{m_k}\to \mathcal{M}_k$ that maps the output of the neural network to a discrete set $\mcal{M}_k\subseteq \mathbb{R}^{m_k}$ of size at most  $2^{R_k}$. That is,
\begin{align}\label{eq:neural_encoding}
\mb{v}_k &= f_k(\bx_k;\theta_k) \in \mbb{R}^{m_k},\ \forall k \in [1:K],\nonumber\\
\bz_k &= \mcal{Q}_k(\mb{v}_k),\ \forall k \in [1:K]. 
\end{align}

Given $\bz = [\bz_1, \bz_2, \cdots, \bz_K]$ as input, the decoder $\mcal{D}$ first applies an initial mapping $\mcal{G}$ that takes $\bz$ to $\mbb{R}^{\widebar{m}}$, where $\widebar{m} = \sum_{k=1}^K m_k$. This serves as a \emph{combiner} for the values $\{\bz_k\}_{k=1}^K$ received from the different encoders. Afterwards, a neural network $g(\cdot;\phi)$, parameterized by $\phi$, is applied on the output of $\mcal{G}$ (see Figure~\ref{fig:model_nn}) before feeding the output $\widehat{\bx}$ to the classifier $\mcal{C}$. Thus, we have
\begin{align}\label{eq:neural_decoding}
\mcal{G} &:\prod_{k=1}^K \mcal{M}_k \to \mbb{R}^{\widebar{m}} \nonumber \\
g(\cdot;\phi) &: \mathbb{R}^{\widebar{m}}\to \mathcal{X}^n \\
\widehat{\bx} &= \mcal{D}(\bz) = g(\mcal{G}(\bz);\phi). \nonumber
\end{align}
Our objective  is to minimize the misclassification loss
\begin{align}
\label{eq:misclassify_loss_emp2}
\min_{\theta_k,\mcal{Q}_k, \mcal{G}, \phi} \frac{1}{N}\sum_{i=1}^N [\mbb{I}(\widehat{y}(\widehat{\mb{x}}^{(i)})\neq y(\bx^{(i)}))].
\end{align}

Instead of minimizing the $0/1$ loss function in~\eqref{eq:misclassify_loss_emp2}, we construct a distribution from the output of the classifier $\mcal{C}$ using a $\tt softmax$ layer, and then apply the cross entropy loss to find the maximum likelihood estimator of $y(\bx)$~\cite{bishop2006pattern} using the samples in dataset $\mcal{T}$. Hence, our objective is to minimize
\begin{equation}
\label{eq:misclassify_loss_emp3}
\min_{\theta_k,\mcal{Q}_k, \mcal{G}, \phi} \mcal{L}^c = \min_{\theta_k,\mcal{Q}_k, \mcal{G}, \phi} \frac{1}{N}\sum_{i=1}^N - \log \left({\tt softmax }[\mathcal{C}(\widehat{\bx}^{(i)})]_{y(\bx^{(i)})}\right),
\end{equation}
where $[{\tt softmax}(\mb{u})]_j =  \exp(u_j)/ \sum_{i=1}^{|\mcal{Y}|} \exp(u_i)$.

We next discuss a challenge in applying standard backpropagation techniques for training our neural networks. Since the classifier $\mcal{C}$ is subdifferentiable, it is possible to compute the gradient of the cross entropy loss in~\eqref{eq:misclassify_loss_emp3} with respect to the decoder parameters $\phi$. 
However, regardless of how the quantizers $\mcal{Q}_k$ are designed, the only subgradient of the quantizers is all zeros.
As a result, it is not possible to apply backpropagation methods~\cite{VQ-VAE,VQ-VAE2,shohat2019deep} to update the encoders parameters $\{\theta_k\}_{k=1}^K$. 
In the following two subsections, we introduce two different approaches for designing the quantizers $\{\mcal{Q}_k\}_{k=1}^K$ and the combiner $\mcal{G}$, and discuss how to incorporate their design in the learning framework of the neural network parameters $\{\theta_k\}_{k=1}^K$ and $\phi$.
\begin{rem}
{\rm 
Note that we do not optimize the classifier $\mcal{C}$ as it is assumed to be pretrained and fixed. However, since the approaches discussed in this section are gradient-based, they can be directly extended to the case where the classifier $\mcal{C}$ is trainable as well, i.e., we can update the parameters of the classifier $\mcal{C}$ as we update the parameters of the networks $f_k$ and $g$ of the encoders and decoders.
}
\end{rem}
\subsection{Discrete distributed neural representation for classification:}\label{sec:neural_REG}
In the first approach, we explicitly design the encoders to produce binary string representations of $\bz_k$. In particular,  for each encoder  $\mcal{E}_k$, the neural network $f_k(\cdot;\theta_k)$ outputs a vector with $R_k$ entries ($m_k=R_k$), and we constrain the range of the elements of this vector  to be in $[-1,1]$. We achieve this by selecting the activations of the last layer of the neural network $f_k(\cdot;\theta_k)$ to be a function that has the range $[-1,1]$ (we used the $\tanh(.)$ function in our numerical evaluation). 
We then simply quantize the output values, by applying the quantizer $\mcal{Q}_k$ as
\begin{equation}
\label{eq:binary_string_quant}
\mcal{Q}_k(\mathbf{u}) = 2*\mathbb{I}(\mathbf{u}\geq 0)-1, \quad \forall k \in [1:K],
\end{equation}
where the indicator function $\mbb{I}$ is applied elementwise.
For the combiner $\mcal{G}$ in the decoder we  simply use an identity function.

\begin{rem}\label{rem:high_quant_noise}
	{\rm 
		As discussed above the $\mcal{Q}_k$ function prevents the backpropagation of the gradient to the encoder network $f_k(\cdot;\theta_k)$. 
		To alleviate this, a straight-through approach is to only use the quantizer blocks in the forward pass and treat them as an identity during backpropagation~\cite{VQ-VAE,shohat2019deep}. This approach works well in some applications~\cite{VQ-VAE}, however, we observed  in our experiments that such an approach prevented the encoder parameters $\{\theta_k\}_{k=1}^K$ from having meaningful gradient updates, and the end-to-end system had a classification performance close to random guessing in the CIFAR10 dataset. In particular, this can happen as when applying the chain rule during backpropagation, we would like to have the derivative $\partial \mcal{L}^c/ \partial \mathbf{v}_k$ to update the parameters $\theta_k$, where $\mathbf{v}_k = f_k(\bx_{\Omega_k};\theta_k)$. Instead, the straight-through approach would use the gradient of a different point in space
		\[
		\frac{\partial \mcal{L}^c}{\partial \mathbf{z}_k} = \frac{\partial \mcal{L}^c}{\partial (\mathbf{v}_k + (\mathbf{z}_k - \mathbf{v}_k))},
		\]
		to update $\theta_k$, where $\bz_k = \mcal{Q}_k(\mb{v}_k)$ (as in Figure~\ref{fig:model_nn}).
		This can be very different from the intended gradient depending on how $\mcal{L}^c$ looks as a function of $\mathbf{z}_k$ and how big is the second term $(\mathbf{z}_k - \mathbf{v}_k)$.
		As an illustration, if one choice of $\theta_k$ results in $v_k = 10^{-6}$ (very close to $0$), it would get quantized to $\bz_k = 1$, resulting in quantization noise $\bz_k - \mb{v}_k = 1-10^{-6}$; if a different $\theta_k$ results in $\mb{v}_k$ being very close to $1$, it would again get quantized to $1$, in this case with negligible quantization noise. Both parameters $\theta_k$ would be updated by the same gradient, even though in the first case, $\mb{v}_k$ was orders of magnitude smaller. Thus, when skipping the quantizer in the backpropagation, the calculated gradients may not be useful if the quantization noise is large.
	}
\end{rem}

\noindent {\bf Regularization for quantization:} Based on the observation in Remark~\ref{rem:high_quant_noise}, we opted to facilitate gradient-based optimization by dropping the quantizers blocks $\{\mathcal{Q}_k\}_{k=1}^K$ during training (both in the forward and backward passes) and instead nudge the network to naturally output values close to quantized ones.
In particular, we penalize the output values that are far from both $-1$ and $+1$, by
introducing an additional term to the loss in \eqref{eq:misclassify_loss_emp3}, termed \textit{quantization loss}, and calculated as
\begin{equation}\label{eq:quantization_loss}
\mcal{L}^{\rm q} = - \frac{1}{KN}\sum_{i=1}^N \sum_{k=1}^{K}\left\|f_k(\bx_{\Omega_k}^{(i)};\theta_k)\right\|_2^2.
\end{equation}
Note that since we choose the activations of the last layer in each encoder to have the range $[-1,1]$,  $\mcal{L}^{\rm q}$ is minimized (achieves the optimal value {$-\sum_k {R_k}/K$}) only when $f_k(\bx_{\Omega_k}^{(i)};\theta_k) \in \{-1,1\}^{R_k} $ $\forall i\in [1:N]$, $k \in [1:K]$. 
Thus, the total training loss becomes
\begin{align}
\label{eq:loss_total_binary}
\mcal{L} = \mcal{L}^{\rm c} + \beta \mcal{L}^{\rm q} =  - \frac{1}{N}\sum_{i=1}^N \left( \log \left({\tt softmax }[\widehat{y}(\widehat{\bx}^{(i)})]_{y(\bx^{(i)})}\right) - \beta \frac{1}{K}\sum_{k=1}^{K}\left\|f_k(\bx_{\Omega_k}^{(i)};\theta_k)\right\|_2^2\right),
\end{align}
where: (i) $\beta$ is a scalar hyperparameter that controls the contribution of $\mcal{L}^{\rm q}$; and (ii) $\mcal{L}^{\rm c}$ is the misclassification loss in \eqref{eq:misclassify_loss_emp3}.

    
For large enough $\beta$, minimizing $\mcal{L}$ can be interpreted as minimizing the classification loss under the constraint that the encoders outputs are very close to $-1/1$, which results in $\|f_k(\bx_{\Omega_k}^{(i)};\theta_k)-\mathcal{Q}_k(f_k(\bx_{\Omega_k}^{(i)};\theta_k))\|_2$ being very small. That is, the outputs without quantization differ by only a small amount from the  outputs with quantization which can be treated as negligible quantization noise during testing.

To illustrate the impact of the \emph{quantization loss} on the distribution of the encoder outputs, Figure~\ref{fig:enc_learn} shows the empirical distribution of the encoders outputs after $50$ training epochs on the CIFAR-$10$ dataset, with and without using $\mcal{L}^{\rm q}$. 
While the classification loss tries to direct the parameters of the encoders and decoder $\{\theta_k\},\phi$ to improve the classification accuracy, the quantization loss adjusts the parameters to push the encoders outputs to be close to $-1/1$.

In the approach discussed in this subsection, we have integrated the quantization during the training phase by modifying the loss function to favor models that have small added noise due to quantization. Instead of modifying the objective function, in the following subsection, we introduce a multi-phase approach, where we first learn continuous representations for classification and then learn a quantizer on these continuous representations using our previously introduced $\tt GBI$ algorithm.

\begin{rem}
{\rm 
Note that the learning approach described in this subsection has computational complexity $O(N*{\tt num\_epochs})$, where {\tt num\_epochs} is typically much smaller than $N$.
}
\end{rem}
\begin{figure}[!htb]
	\begin{subfigure}{0.5\textwidth}
		\centering
		\includegraphics[scale=0.44]{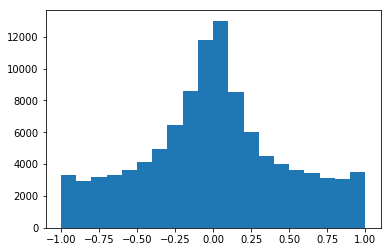} 
		\caption{without quantization loss}
	\end{subfigure}\quad
	\begin{subfigure}{0.5\textwidth}
		\centering
		\includegraphics[scale=0.44]{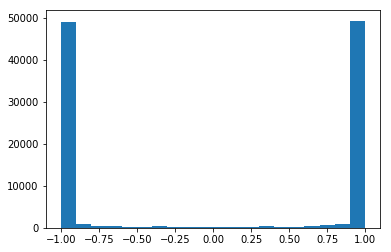}
		\caption{ with quantization loss}
	\end{subfigure}
	\caption{Effect of the quantization loss on the distribution of the decoder inputs after training for 50 epochs on the CIFAR10 dataset.}
	\label{fig:enc_learn}
\end{figure}

\subsection{Distributed neural representation using {\tt GBI}:}\label{sec:neural_GBI}
The main idea in this approach is to use the $\tt GBI$ algorithm to design the quantizers $\{\mcal{Q}_k\}_{k=1}^K$ and the initial decoder $\mcal{G}$.
We first use the neural network $f_k(\cdot;\theta_k)$ to map the features of node $k$ from $\mathbf{x}_{\Omega_k}\in \mathbb{R}^{|\Omega_k|}$ to $\mathbf{v}_k\in \mathbb{R}^{m_k}$.
We select $m_k$ to be as small as possible while maintaining a good classification accuracy; that is, the networks $f_k(\cdot;\theta_k)$ essentially perform dimensionality reduction at the encoders before applying the quantization step. We then apply  $\tt GBI$ on the output of the encoder neural network $f_k$.
The main benefit is that, by decreasing the number of dimensions of the input to $\tt GBI$ algorithm, we alleviate the complexity of $\tt GBI$, that grows with the number of dimensions (see Section~\ref{sec:GBI}).


Given that the neural network $f_k(\cdot;\theta_k)$ is potentially a universal function approximator~\cite{csaji2001approximation}, it is not difficult to see that even using a naive uniform quantizer $\mathcal{Q}_k$, we could potentially implement any encoder $\mathcal{E}_k = \mathcal{Q}_k(f_k(\cdot;\theta_k)) :\mathbb{R}^{|\Omega_k|}\to \mathcal{M}_k$. Hence, ideally, the choice of the quantizer $\mathcal{Q}_k$ should not play a significant role. However, due to the fact that neural networks tend to work well only if the initialization is close to a good solution, the choice of the quantizer becomes important. In the following, we propose a method
that operates in three phases:

\par \noindent {\bf Phase 1.} We first train the encoders and decoder neural networks without any quantization units (i.e., without $\mcal{Q}_k$ and $\mcal{G}$) until we get classification accuracy that is close to the classifier's accuracy. Note that for $m_k\geq |\Omega_k|$, we can reconstruct the classifier accuracy. Effectively, in this step, we are following the example structure shown in Figure~\ref{fig:model_nn}, assuming the blocks $\mcal{Q}_k$ and $\mcal{G}$ are identities.

\par \noindent {\bf Phase 2.} With the parameters $\theta_k, \phi$ learned in {\bf Phase 1}, we now design the quantization components $\{\mathcal{Q}_k\}_{k=1}^K$ and $\mcal{G}$ based on the outputs $\{\mathbf{v}_k\}$ of the neural networks $f_k(\cdot;\theta_k)$. If the quantizer maps data points that have different labels to the same quantized value, the quantized value cannot be used to classify the points correctly. Hence, the objective of the quantizer is to map points that have different labels to different quantized values. We do this by introducing boundaries in the space using our proposed $\texttt{GBI}$ quantizer, described earlier in Section \ref{sec:GBI}.

\par \noindent{\bf Phase 3.} Finally, we continue training the encoders and decoder neural networks ($f_k$ and $g$) with the quantizers designed in {\bf Phase 2}.  To do so we skip the quantizers blocks in the back propagation and consider them only in the forward pass. We observed in our experiments that this skip does not cause the network to behave randomly as the initialization is designed carefully. The parameters learned in {\bf Phase 1} act as initializations for $\theta_k, \phi$ in this phase. Phase 3 enables to fine tune the network parameters given that we have already learned the quantizer components earlier in {\bf Phase 2}.

\section{Experimental Evaluation}
\label{sec:experiments}
In this section, we present various experimental results, comparing the behaviour of our proposed distributed quantization approaches with quantization approaches for reconstruction.
We find that tailoring the quantizers to the classification task can offer significant savings: we can achieve more than a factor of two reduction in terms of the number of bits communicated, for the same classification accuracy.
Moreover, our algorithms retain reasonable classification performance even when constrained to use a very small number of bits per encoder; for instance, for $2$ bits per encoder, we achieve approximately two to four times the classification accuracy of alternative approaches.

Additionally, we also compare to centralized quantization approaches for classification and show that despite our distributed setup, we are still able to achieve a competitive performance in terms of classification accuracy.

\subsection{Performance on Electromyography sensor measurements:}
We start with experiments on a dataset of surface electromyographic (sEMG) signals~\cite{lobov2018latent}. Each data point represents measurements recorded from 8 sensors that are used to differentiate between 6 different hand gestures. For our classifier $\mcal{C}$, we use a Multi-Layer Perceptron (MLP)~\cite{bishop2006pattern} architecture with fully connected layers of the form $8-100-200-200-200-6$ and ReLu activations. The classifier was pretrained on an unquantized training set of $15,345$ measurements, and yielded a baseline accuracy of $98.66\%$ on a test set of $6,578$ measurements. 

For our distributed quantization framework, we assume that we have $K=4$ encoders, where each encoder has access to measurements from only two of the sensors (i.e., we have four feature groups each consisting of two features). We use MLPs for our encoders and decoder, while the quantizers are either trained using the quantization loss regularization {\bf (NN-REG)} or with the GBI algorithm {\bf (NN-GBI)} as described in Section~\ref{sec:neural}. The hidden layers structures of encoders/decoders, and the hyper parameters (learning rate and regularization weight $\beta$) are described in Appendix~\ref{appendix:net_structure}.

\par{\bf Comparison with quantization for reconstruction.}
We demonstrate that our approaches achieve competitive classification results with smaller number of bits as compared to distributed approaches aiming at reconstructing the input. We compare against the $k$-means algorithm~\cite{macqueen1967some} as a representative for unsupervised vector quantization algorithms. In particular, in the sEMG dataset, each $k$-means encoder maps a point in $\mbb{R}^2$ to the nearest centroid point among $2^R$ choices. The decoder is treated as an identity in this case, that passes its input $\bz$ vector to the classifier. 

The results are shown in Table~\ref{tbl:accuracies_EMG}. We see that our approaches outperform the unsupervised distributed quantization. For example, using $4$ bits for each encoder, we can achieve a classification accuracy of $> 95\%$ while the unsupervised approach achieves a performance of $66\%$.
\begin{table}
	\centering
	\begin{tabular}{@{}lccccc@{}}
		\toprule
		& \multicolumn{5}{c}{\bf bits per encoder ($R$)}  \\ 
		\cmidrule{2-6}
		&\bf 1    &\bf 2    &\bf 3    &\bf 4    &\bf 5    \\
		\bottomrule
		{\bf $k$-means}~\cite{macqueen1967some} & 18.77\% & 44.18\% & 56.47\% & 66.71\% & 75.19\% \\
		\bf NN-REG						        & 54.50\% & 63.04\% & 82.90\% & 94.72\% & 97.73\% \\
		\bf NN-GBI      						& 55.49\% & 72.35\% & 91.12\% & 97.30\% & 98.21\% \\ 
		\bottomrule
	\end{tabular}
	\caption{Correct classification percentage on the sEMG test set. Each system uses $R$ bits per encoder, $K=4$ encoders and a pretrained classifier with 98.66\% accuracy.}
	\label{tbl:accuracies_EMG}
	\vspace{-2em}
\end{table}


\par{\bf Comparison with learning vector quantization for classification.}
To benchmark the performance of our distributed quantization system for classification against centralized approaches, a natural candidate for comparison is the centralized Generalized Learning Vector Quantization approach (GLVQ)~\cite{sato1996generalized}. In this case, the output of the algorithm is a Voronoi tessellation in the space, where each centroid is now associated with a class. Thus, by mapping a point in space to its nearest centroid, a classification is also performed by picking the class associated to the selected centroid. We compare the performance of our distributed quantization approachs against the quantizer-classifier learned by the LVQ3 centralized quantizer~\cite{sato1996generalized} with learning rate $10^{-4}$ for $200$ epochs.

Since our distributed quantization system with $K=4$ encoders uses $2$ bits per encoder, we allow LVQ3 to use $8$ bits (i.e., $64$ centroids) to keep the total number of bits across the nodes constant. Our {\bf NN-GBI} approach, yielded $71.59\%$ classification accuracy, while the centralized LVQ classifier gave an accuracy of $75.53\%$.

Although LVQ gives a better classification accuracy, the learned Voronoi boundaries are not decomposable to be applied on distributed nodes. In particular, in the described setting, when inspecting the values of the centroids learned by the LVQ algorithm, we found that although $2^8$
centroids are used in $\mbb{R}^8$, restricting the values of the centroids to any one dimension of the 8, gave $64$ distinct values which would require each of the $4$ encoders to at least use 8 bits to represent these quantized values. Recall that from Table~\ref{tbl:accuracies_EMG}, we are able to achieve much higher accuracies than $75\%$, when only $5$ bits are used at each encoder.

\subsection{Performance on CIFAR10 images:}
In this set of experiments, we evaluate the performance of our proposed algorithms on the CIFAR10 dataset, where each input $\bx$ is a $32\times 32 \times 3$ image. Each image in the CIFAR10 dataset is associated with one of $10$ classes.  
We assume a distributed quantization system with $K=4$ distributed encoders, that each have access to a quadrant of the image. For the classifier, we use a pretrained VGG-13 classifier~\cite{SimonyanZ14a} with $94.27\%$ accuracy on the CIFAR10 test dataset.

\par{\bf Comparison to VQ-VAE.} The VQ-VAE~\cite{VQ-VAE} framework  is used to learn discrete neural representations of a dataset for reconstruction. We compare against this framework implemented both in a centralized and a distributed fashion. In particular, for the centralized VQ-VAE, a single encoder has access to the full image. We use the same VQ-VAE network structure from~\cite{VQ-VAE} for the CIFAR10 dataset and ensure that the total number of bits used by the VQ-VAE encoder is $4$ times what our system would use for a single encoder.
In the distributed setting, a VQ-VAE encoder is applied on each image quadrant and then a common decoder is used for reconstruction.
VQ-VAE structures were trained with $2\times 10^{-4}$ learning rate, $200$ epochs and $64$ batch size. 

\begin{table}
	\centering
	\begin{tabular}{@{}lccccc@{}}
		\toprule
		& \multicolumn{5}{c}{\bf bits per encoder ($R$)}  \\ 
		\cmidrule{2-6}
		&\bf 1    &\bf 2    &\bf 3    &\bf 4    &\bf 5    \\
		\bottomrule
		{\bf VQ-VAE (centralized)} & 13.82\% & 14.36\% & 15.87\% & 17.52\% & 18.18\% \\
		{\bf VQ-VAE (distributed)} & 10.12\% & 10.46\% & 11.03\% & 11.61\% & 12.28\% \\
		\bf NN-REG        & 48.63\% & 63.32\% & 68.07\% & 73.43\% & 78.12\% \\
		\bf NN-GBI      & 48.33\% & 60.88\% & 65.16\% & 71.57\% & 81.18\% \\ 
		\bottomrule
	\end{tabular}
	\caption{Correct classification percentage on the CIFAR10 test set. All distributed systems use $R$ bits per encoder, $K=4$ encoders.
		The centralized VQ-VAE system uses $4R$ bits at the encoder.
		The classifier is a pretrained VGG-13 with 94.27\% accuracy.}
	\label{tbl:accuracies_CIFAR}
	\vspace{-1em}
\end{table}

The results are summarized in Table~\ref{tbl:accuracies_CIFAR}. 
We find that although VQ-VAE has shown great success in reconstructing images from discrete representations, it does not perform  well with a low number of bits even in the centralized  case. 
To get classification accruacy of around $50\%$, the centralized VQ-VAE  required $200$ bits (equivalent to $50$ bits/encoder in the distributed setting), while our algorithms could get more than  $70\%$ accuracy with $3$ bits per encoder.


\section{Conclusions}
\label{conclusion}
In this paper, we introduced the problem of data-driven distributed data quantization for classification.
We proved that in many cases, designing an optimal quantization system is an NP-hard problem  that is also hard to approximate. 
For a case that is not NP-hard, we proposed an optimal polynomial-time algorithm for designing the quantizer under some structural restrictions. 
For the NP-hard cases, we proposed a polynomial time greedy approach and two learning based approaches. 
Numerical results on the sEMG and CIFAR10 datasets indicate that 
tailoring the quantizers to the classification task can offer significant savings: more (and in some cases much more) than  a factor of two reduction in the number of bits communicated, for the same classification accuracy. Moreover, our algorithms retain reasonable classification performance even when constrained to use a very small number of bits.

\appendices
\section{Counter-Example for Adaboost inspired quantization}
\label{sec:xor_example}
\begin{figure}[!ht]
	\centering
	\includegraphics[width=0.5\textwidth]{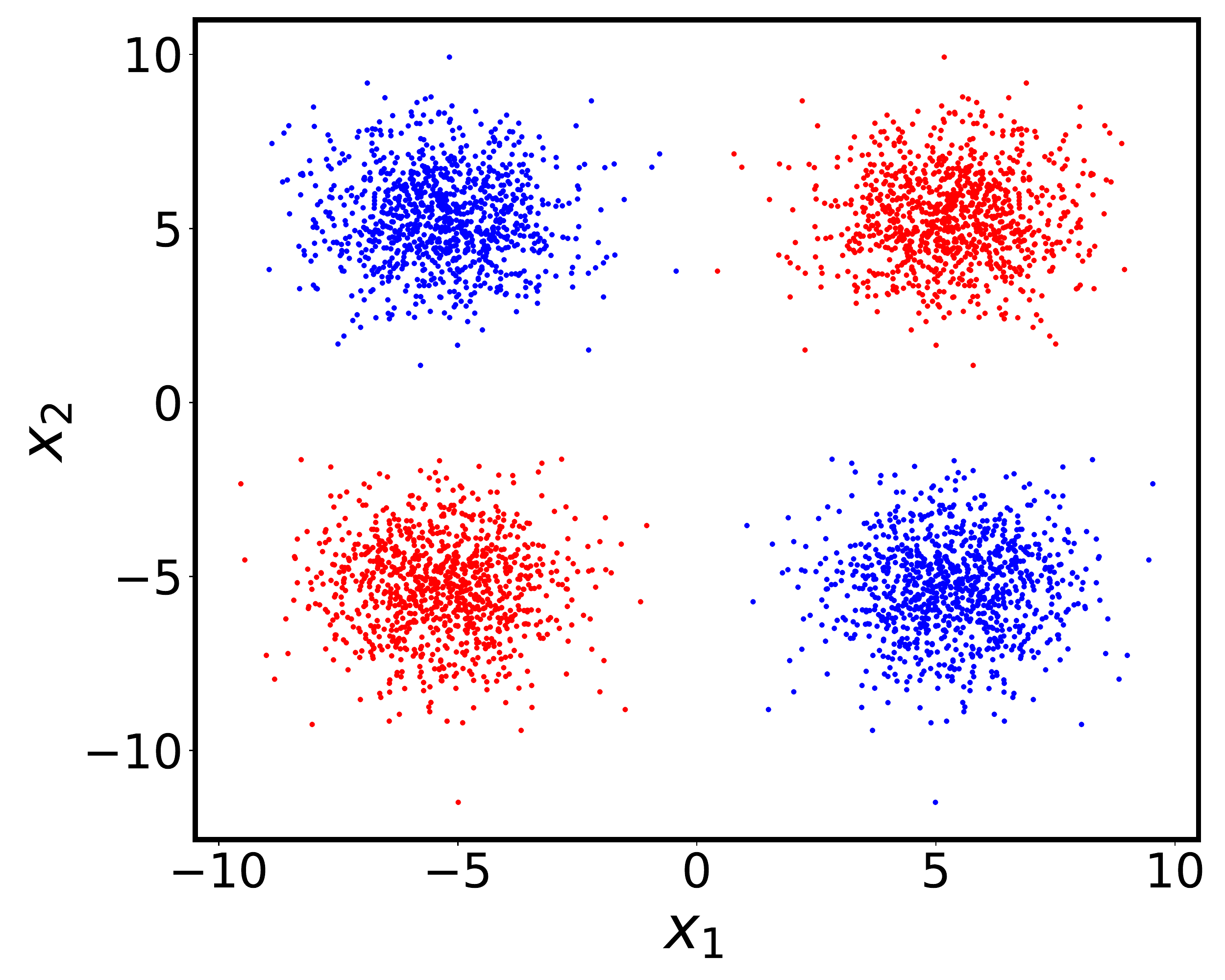}
	\vspace{-1em}
	\caption{Xor-like example where Adaboost inspired quantization would fail.}
	\label{fig:xor_example}
\end{figure}

\newpage
\section{NP-hardness and hardness of approximation} \label{hardness_proof}

In this appendix we prove Theorem \ref{hardness}.
\subsection{NP-hardness of \ref{i} and \ref{ii}}
In this subsection, we prove that~\ref{i}, \ref{ii} are NP-hard. We start with \ref{ii}. Since \ref{ii} is a special case of \ref{i}; it follows that \ref{i} is NP-hard. We prove hardness of \ref{ii} by reduction from the Chromatic Number problem. In particular, we show that any instance of the Chromatic Number problem can be reduced to an instance of problem \ref{ii} in polynomial time. The decision version of the Chromatic Number problem is on Karp's list of NP-complete problems \cite{karp1972reducibility}.

Let us consider an undirected graph $G$. We denote the set of vertices and edges of $G$ by $V, E$ respectively, and assume that the vertices are labeled by numbers $1,2,...,|V|$. Since the graph is undirected, we assume without loss of generality that the vertices pair corresponding to each edge is ordered such that if $(q_1,q_2)\in E$, then $q_1>q_2$. We will see that this assumption, ensures that we can construct a linearly separable dataset as required in problem \ref{ii}.

We construct two matrices $\{F_i\}_{i=0}^1, F_i\in \mathbb{R}^{|E|\times |V|}$ that represent the set of edges $E$, by Algorithm \ref{alg}.
\begin{algorithm} \caption{Incidence matrices} \label{alg}
	\textbf{0:} Initialize the entries of $F_i\in \mathbb{R}^{|E|\times |V|}, i\in \{0,1\}$ with all zeros. $k=1$.\\
	$\forall (q_1,q_2)\in E$ , do the following two steps:\\
	\hspace*{4mm} \textbf{1:}	Put $[F_0]_{kq_1}=1,[F_1]_{kq_2}=1$.\\
	\hspace*{4mm} \textbf{2:} k=k+1.
\end{algorithm}
We can think of the matrices $\{F_i\}_{i=0}^1$ as a decomposed version of the incidence matrix, where for each edge, one endpoint is represented in $F_0$ and the other endpoint is represented in $F_1$. As an illustrative example, we consider the graph in Figure~\ref{graph_1}. The corresponding matrices are given by 

\begin{equation}\label{dist_1}
F_0=\begin{bmatrix}
0 & 0 & 1 & 0 & 0 \\
0 & 0 & 0 & 1 & 0 \\
0 & 0 & 0 & 0 & 1 \\
0 & 0 & 0 & 1 & 0 \\
0 & 0 & 0 & 1 & 0 \\
0 & 0 & 0 & 0 & 1 \\
0 & 0 & 0 & 0 & 1 \\
0 & 1 & 0 & 0 & 0
\end{bmatrix},
F_1=\begin{bmatrix}
1 & 0 & 0 & 0 & 0 \\
1 & 0 & 0 & 0 & 0 \\
1 & 0 & 0 & 0 & 0 \\
0 & 1 & 0 & 0 & 0 \\
0 & 0 & 1 & 0 & 0 \\
0 & 0 & 1 & 0 & 0 \\
0 & 0 & 0 & 1 & 0 \\
1 & 0 & 0 & 0 & 0
\end{bmatrix}.
\end{equation}

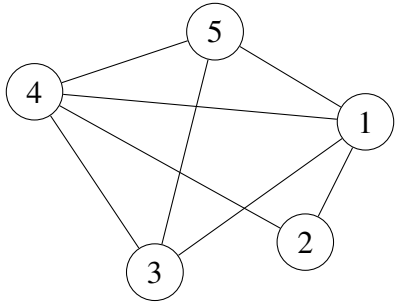
\begin{figure}[!ht]
	\center
	\begin{tikzpicture}
	[scale=.8,auto=left,every node/.style={circle,fill=none,draw}]
	\node (n4) at (5,8)  {4};
	\node (n5) at (8,9)  {5};
	\node (n1) at (10.5,7.5) {1};
	\node (n2) at (9.5,5.5)  {2};
	\node (n3) at (7,5)  {3};
	
	\foreach \from/\to in {n1/n2,n1/n3,n1/n4,n1/n5,n2/n4,n3/n4,n3/n5,n4/n5}
	\draw (\from) -- (\to);
	
	\end{tikzpicture}
	\caption{Graph with edges represented by the matrices in \eqref{dist_1}.}
	\label{graph_1}
\end{figure}
If two vertices are colored with the same color, we update the matrices $\{F_i\}_{i=0}^1$ by replacing the columns that correspond to the vertices colored with the same color, with their sum. For example, if vertices $1,2$ are assigned the same color, we update each matrix $F_i$ by replacing the first two columns with their sum, which results in
\begin{equation}\label{updated matrices}
F_0^{'}=\begin{bmatrix}
0 & 1 & 0 & 0 \\
0 & 0 & 1 & 0 \\
0 & 0 & 0 & 1 \\
0 & 0 & 1 & 0 \\
0 & 0 & 1 & 0 \\
0 & 0 & 0 & 1 \\
0 & 0 & 0 & 1 \\
1 & 0 & 0 & 0
\end{bmatrix},
F_1^{'}=\begin{bmatrix}
1 & 0 & 0 & 0 \\
1 & 0 & 0 & 0 \\
1 & 0 & 0 & 0 \\
1 & 0 & 0 & 0 \\
0 & 1 & 0 & 0 \\
0 & 1 & 0 & 0 \\
0 & 0 & 1 & 0 \\
1 & 0 & 0 & 0
\end{bmatrix}.
\end{equation}

We notice that {\bf a coloring is valid} (no two vertices connected with an edge are assigned the same color) if and only if the updated matrices satisfy: 
\begin{equation}\label{eq:property}
\forall k\in [1:|E|],\ \forall q\in [1:V'],\ \forall i \in \{0,1\}\quad : \quad  [F_i^{'}]_{kq}\neq 0\implies [F_{1-i}^{'}]_{kq}= 0,
\end{equation}
where $V'$ is the number of columns of the matrix $F_0^{'}$ or $F_1^{'}$. In the above example, the property in \eqref{eq:property} is not satisfied since, $[F_i^{'}]_{80}= 1, [F_{1-i}^{'}]_{81}= 1$. This is because the vertices $1,2$, which are assigned the same color, are connected with an edge.

Hence, the Chromatic Number of the graph is the minimum number of columns of matrices $\{F_i^{'}\}_{i=0}^1$ that satisfy the property in \eqref{eq:property} and are constructed according to the following rules:
\begin{itemize}
	\item Any set of columns in the matrix $F_i$ can be replaced by their sum, $i=0,1$.
	\item If the set of columns indexed by $\mathcal{I}$ in $F_i$ are replaced with their sum, then the set of columns indexed by $\mathcal{I}$ in $F_{1-i}$ are replaced with their sum, $i=0,1$, i.e., exactly the same operations done on $F_0$ are done on $F_1$ and vice versa. 
\end{itemize}

The next step is to consider an instance from the problem \ref{ii}, and show that it is equivalent to the problem of finding the minimum number of columns of the matrices $\{F_i^{'}\}_{i=0}^1$. To that end, we consider a dataset, with two classes, namely $\mcal{Y} = \{0,1\}$, and two features $x_1,x_2$. The dataset is constructed based on the matrices $\{F_i\}_{i=0}^1$ by Algorithm \ref{alg_red_color}.

\begin{algorithm} \caption{Reduction from vertex coloring problem} \label{alg_red_color}
	\textbf{0:} Start with $k=1$.\\
	$\forall (q_1,q_2)\in E$ , do the following two steps:\\
	\hspace*{4mm} \textbf{1:} Put a training point that belong to class 0 at $x_1=q_1,x_2=\frac{q_1+q_2}{2}$,\\
	\hspace*{9mm} and a training point that belong to class 1 at $x_1=q_2,x_2=\frac{q_1+q_2}{2}$.\\
	\hspace*{4mm} \textbf{2:} k=k+1.
\end{algorithm}
Note that under the assumption $q_1>q_2$, the point $(x_1,x_2) = (q_1,\frac{q_1+q_2}{2})$ lies on the right side of the line $x_1=x_2$, while the point $(x_1,x_2)=(q_2,\frac{q_1+q_2}{2})$ lies on the left side of the line $x_1=x_2$. Hence, all the points that belong to class 0 lie on the right side of the line $x_1=x_2$, and all the points that belong to class 1 lie on the left side of the line $x_1=x_2$. That is, the constructed dataset is linearly separable by the line $x_1=x_2$.

For the constructed dataset, we want to answer the following question: for $R_2\to \infty$, what is the minimum $R_1$ for which $\mcal{L}(\mathbf{\mathbf{\mathcal{E}}},\mcal{D},\mcal{T})<\frac{1}{2}$?  

As an example, consider the matrices $\{F_i\}_{i=0}^1$ in \eqref{dist_1}, the constructed dataset is given in Figure~\ref{fig_1}. 
\begin{figure}[h!]
	\centering
	\includegraphics[width=0.7\textwidth]{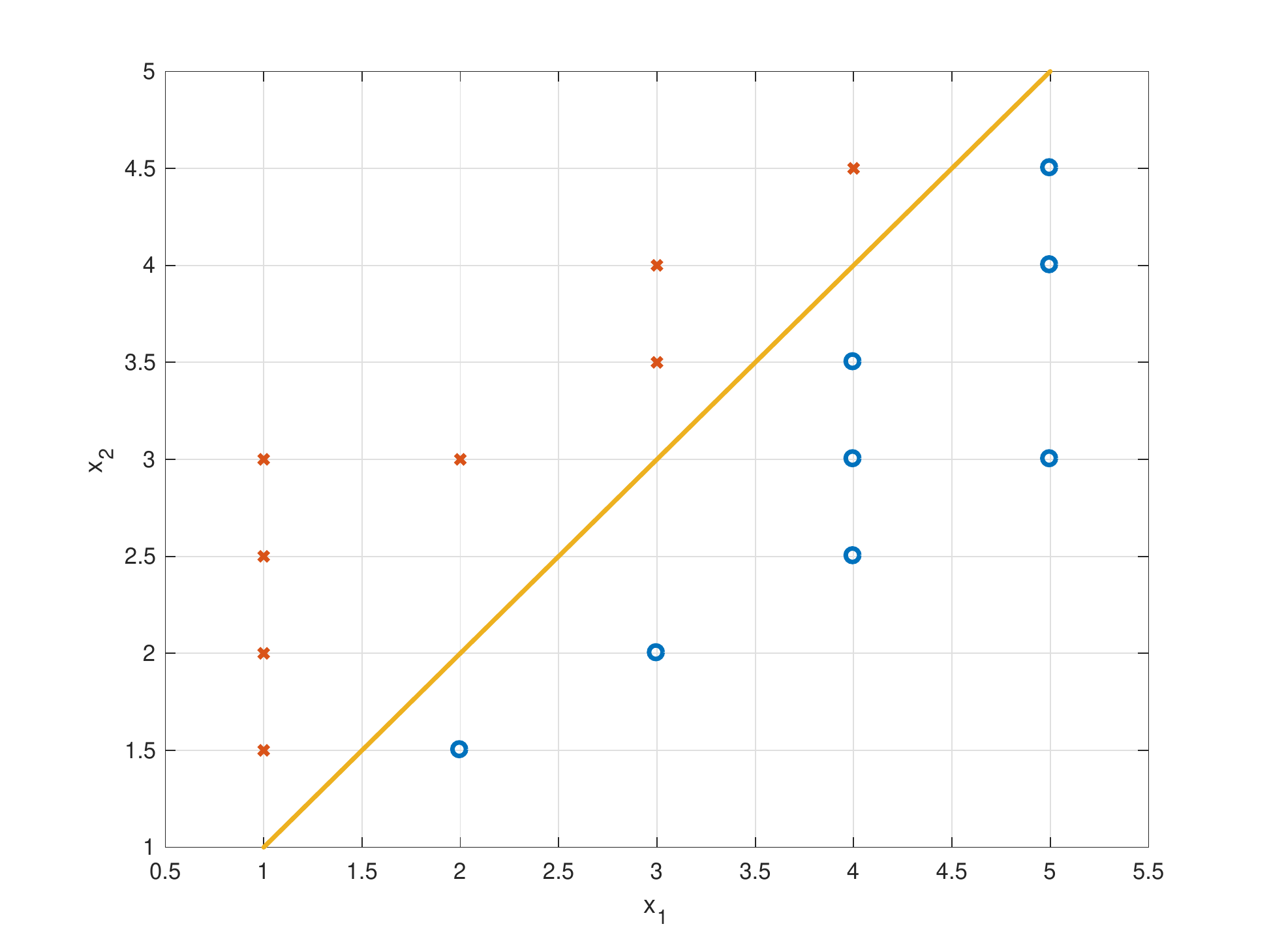}
	\caption{Training data corresponding to $\{F_i\}$ in \eqref{dist_1}.}
	\label{fig_1} 
\end{figure}

Assume that we want to find the minimum $R_1$ for which $\mcal{L}(\mathbf{\mathbf{\mathcal{E}}},\mcal{D},\mcal{T})<\frac{1}{2}$. If $2^{R_1}\geq 5$, then we do not need to do quantization and can send $x_1$ as it is. In this case we have $\mcal{L}(\mathbf{\mathbf{\mathcal{E}}},\mcal{D},\mcal{T})=0$. Now, assume that $2^{R_1}=4$, then we have only $4$ values to send to represent $x_1$, i.e., $\mcal{E}_1(x_1)\in [1:4]$. Hence, the quantizer has to map two different values of $x_1$ to the same quantized value, i.e., $\exists x_1^{(1)},x_1^{(2)}: x_1^{(1)}\neq x_1^{(2)}, \mcal{E}_1(x_1^{(1)})=\mcal{E}_1(x_1^{(2)})$. The matrices $\{F_i^{'}\}$ are constructed from $\{F_i\}$ such that if the encoder maps a set of $x_1$ values to the same encoded value, we replace the corresponding columns with their sum. For instance, assume that the encoder maps the values $x_1=1, x_1=2$ to the same quantized value. Based on this, we update each matrix $F_i$ by replacing the first two columns with their sum, which results in the matrices in \eqref{updated matrices}. Since this encoder maps the two points $(1,1.5),(2,1.5)$, which belong to different classes, to the same encoded value, we have $\mcal{L}(\mathbf{\mathbf{\mathcal{E}}},\mcal{D},\mcal{T})= \sum_{k=1}^{l_2}\sum_{q=1}^{l_1-1} \min_j\{[F_j^{'}]_{kq}\}=1>\frac{1}{2}$. So, this encoder does not satisfy $\mcal{L}(\mathbf{\mathbf{\mathcal{E}}},\mcal{D},\mcal{T})<\frac{1}{2}$. The reason for this is that the updated matrices do not satisfy: $\forall k\in [1:|E|] \forall q\in [1:|V'|] \forall i \in \{0,1\} \ [F_i^{'}]_{kq}\neq 0\implies [F_{1-i}^{'}]_{kq}= 0$. 
Note that the matrices constructed by Algorithm \ref{alg} satisfy the following properties
\begin{itemize}
	\item Every row in $F_i$ has exactly one non-zero entry.
	\item If $[F_i]_{kq}\neq 0$, then $[F_{1-i}]_{kq}= 0$.
	\item All the non-zero entries have value $1$.
\end{itemize}
It is easy to observe that for all matrices satisfying the three properties mentioned above, $\mcal{L}(\mathbf{\mathbf{\mathcal{E}}},\mcal{D},\mcal{T})<\frac{1}{2}$ if and only if the encoder $\mathbf{\mathcal{E}}$ satisfies: 
\begin{equation} \label{prop_1}
\begin{aligned}
&\forall q_1 \forall q_2:q_1,q_2\in \text{dom}(\mcal{E}_1), q_1\neq q_2, \mcal{E}_1(q_1)=\mcal{E}_1(q_2)\\
&[\forall k\in [1:|E|] \forall i \in \{0,1\}\ \left( [F_i]_{kq_1}+[F_i]_{kq_2}\neq 0\implies [F_{1-i}]_{kq_1}+[F_{1-i}]_{kq_2}=0\right) ].
\end{aligned}
\end{equation}
Hence, the $\min\{2^{R_1}|\mcal{L}(\mathbf{\mathbf{\mathcal{E}}},\mcal{D},\mcal{T})<\frac{1}{2}\}$ is the minimum number of columns of matrices $\{F_i^{'}\}_{i=0}^1$ that satisfy the property in \eqref{eq:property} and are constructed according to the following rules:
\begin{itemize}
	\item Any set of columns in the matrix $F_i$ can be replaced by their sum, $i=0,1$.
	\item If the set of columns indexed by $\mathcal{I}$ in $F_i$ are replaced with their sum, then the set of columns indexed by $\mathcal{I}$ in $F_{1-i}$ are replaced with their sum, $i=0,1$, i.e., exactly the same operations done on $F_0$ are done on $F_1$ and vice versa. 
\end{itemize}
This shows that $\mathcal{X}(G)=\min\{2^{R_1}|\mcal{L}(\mathbf{\mathbf{\mathcal{E}}},\mcal{D},\mcal{T}_G)<\frac{1}{2}\}$, where $\mathcal{X}(G)$ is the chromatic number of the graph $G$, and $\mathcal{T}_G$ is the dataset constructed by Algorithm \ref{alg_red_color}.

Note that the maximum number of $x_1$ values that are encoded to the same value is $|V|$, i.e., $\mathcal{E}_1(\bz_1)$ is the union of at most $|V|$ $\mathbb{R}$-bins. Hence, $r$ is chosen to be $r=|V|$. This concludes the proof that problem \ref{ii} is NP-hard.

\subsection{NP-hardness of \ref{iii}:}
\newcommand{\TG}{\mcal{T}_G}
\newcommand{\TV}{\mcal{T}_{v_1,v_2}}
\newcommand{\qa}{\mb{q}_{v_1}}
\newcommand{\qb}{\mb{q}_{v_2}}
In this subsection, we prove that~\ref{iii} is NP-hard.
For reference, we restate the statement of~\ref{iii} below:

\par {(\bf P3') : } For $n = 2$ features, $|\mcal{Y}| = 2$ classes, find the minimum $R_1=R_2$ bits for which $\mcal{L}(\mathbf{\mathbf{\mathcal{E}}},\mcal{D},\mcal{T})<p_e$, assuming that $\mcal{E}_k^{-1}(\bz_k)$ is a single $\mathcal{X}_{\Omega_k}$-bin.

To prove that~\ref{iii} is NP-hard, we show that the maximum Balanced Complete Bipartite Subgraph (BCBS) problem~\cite{alon1994algorithmic} can be reduced to~\ref{iii} in polynomial-time. The maximum BCBS problem is defined below.
\begin{defin} (maximum BCBS)
	{\rm 
		Given a balanced bipartite graph $G = (V_1,V_2,E)$: find the maximum size, in terms of number of vertices, of a balanced bi-clique in $G$\footnote{By a balanced bipartite graph $G$, we mean that $|V_1| = |V_2|$.}.
	}
\end{defin}
The maximum BCBS problem is known to be NP-hard as proved in~\cite{alon1994algorithmic}. 

We start by considering the case where the minimum degree in $G$
is nonzero and then address the case, where there are vertices with zero degree afterwards.

Our first step in the reduction is to construct a dataset $\mcal{T}_G = \{(\bx,y(\bx))\}$ based on the graph $G$.
To do so, we model each pair of vertices $(v_1,v_2) \in [1:V_1(G)]\times [1:V_2(G)]$ as a gadget dataset $\mcal{T}_{v_1,v_2}$ of labeled points $(\bx,y(\bx))$, where $v_1,v_2$ are the indices of the vertex pair.
The dataset $T_{v_1,v_2}$ can take one of two choices ($\mcal{T}_{v_1,v_2}^{(1)}$ or $\mcal{T}_{v_1,v_2}^{(0)}$) depending on whether $(v_1,v_2) \in E$ or not, respectively. In particular, these two gadget choices are defined below
\begin{figure}
	\centering
	\includegraphics[width=0.7\textwidth]{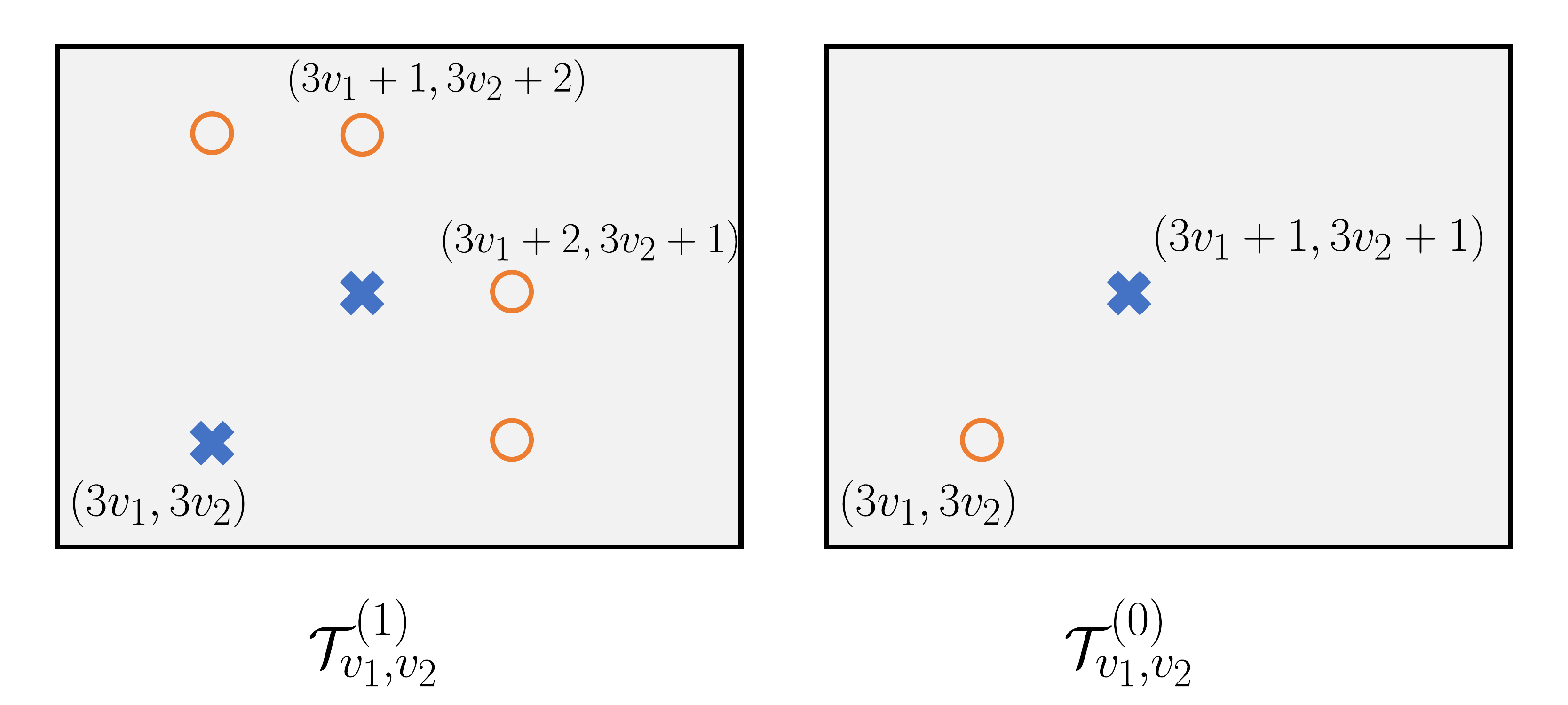}
	\caption{Structures of gadget datasets $T_{v_1,v_2}^{(1)}$ and $T_{v_1,v_2}^{(0)}$.}
	\label{T_v_gadgets}
\end{figure}
\begin{align}\label{eq:gadget_1}
\mcal{T}_{v_1,v_2}^{(1)} =& \left\{([3v_1,3v_2],1),\ ([3v_1{+}1,3v_2{+}1],1),\right.\nonumber\\
&	  ([3v_1{+}2,3v_2],0),\ ([3v_1{+}2,3v_2{+}1],0),\nonumber\\ 
&	  \left.([3v_1,3v_2{+}2],0),\ ([3v_1+1,3v_2{+}2],0)\right\}	 \quad {\rm if}\ (v_1,v_2) \in E,\\
\nonumber
\end{align}
\begin{align}
\label{eq:gadget_2}
\mcal{T}_{v_1,v_2}^{(0)} =& \left\{([3v_1,3v_2],0),\ ([3v_1+1,3v_2+1],1)\right\}		                             \quad {\rm if}\ (v_1,v_2) \not\in E.
\end{align}
Figure~\ref{T_v_gadgets} shows an illustration of the two possible versions of $\mcal{T}_{v_1,v_2}$.
The constructed dataset $\mcal{T}_G$ is the union of all gadget dataset $\mcal{T}_{v_1,v_2},\ \forall v_1,v_2$, i.e.,
\[
\mcal{T}_G = \bigcup_{(v_1,v_2) \in [1:V_1(G)]\times[1:V_2(G)]} \mcal{T}_{v_1,v_2}.
\]
\begin{figure}[!ht]
	\centering
	\begin{subfigure}{0.45\textwidth}
		\includegraphics[width=1\textwidth]{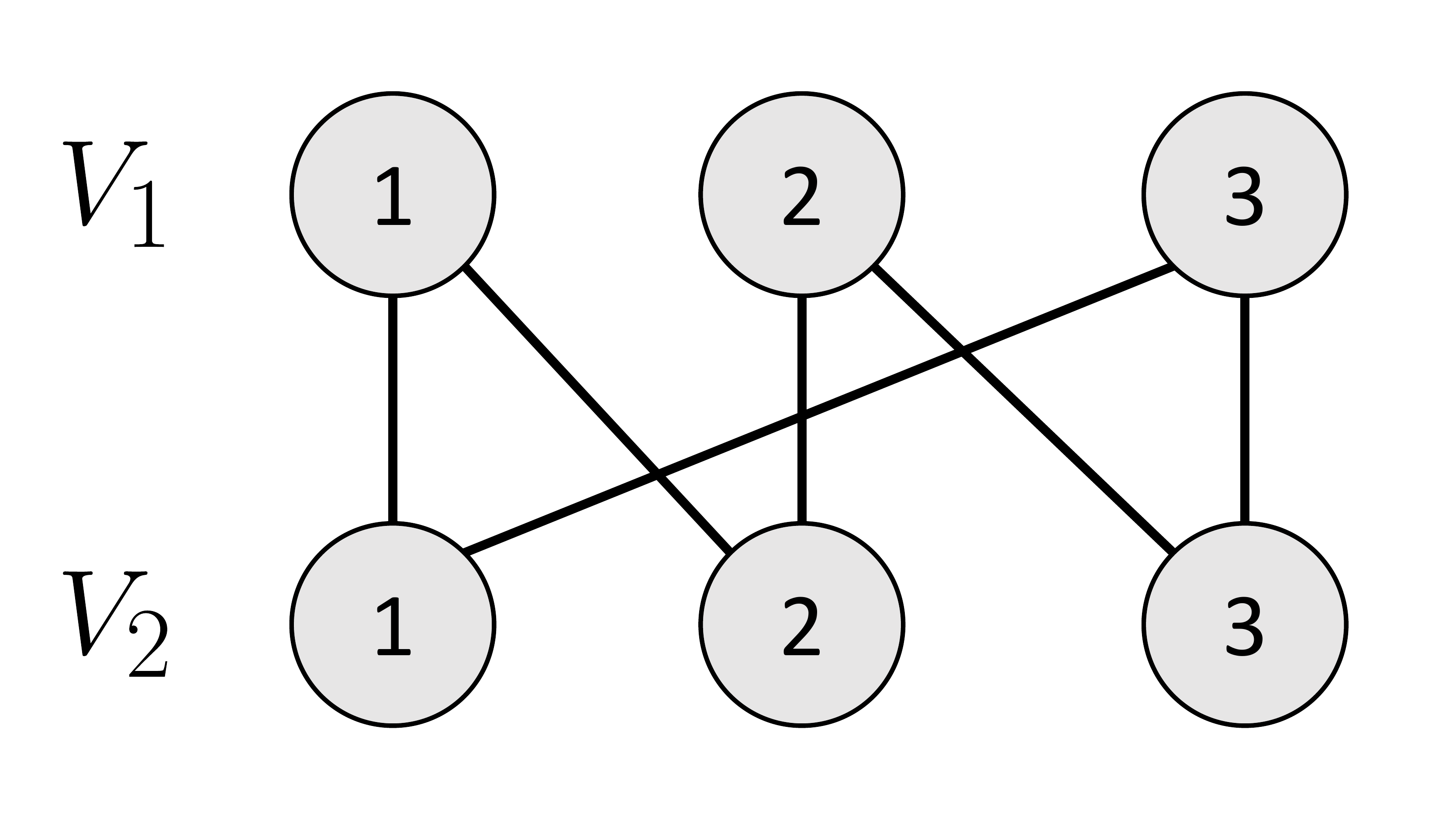}
		\vspace{1em}
		\caption{Example graph $G$.}
		\label{graph_bi_example_fig1}
	\end{subfigure}
	\begin{subfigure}{0.45\textwidth}
		\includegraphics[width=1\textwidth]{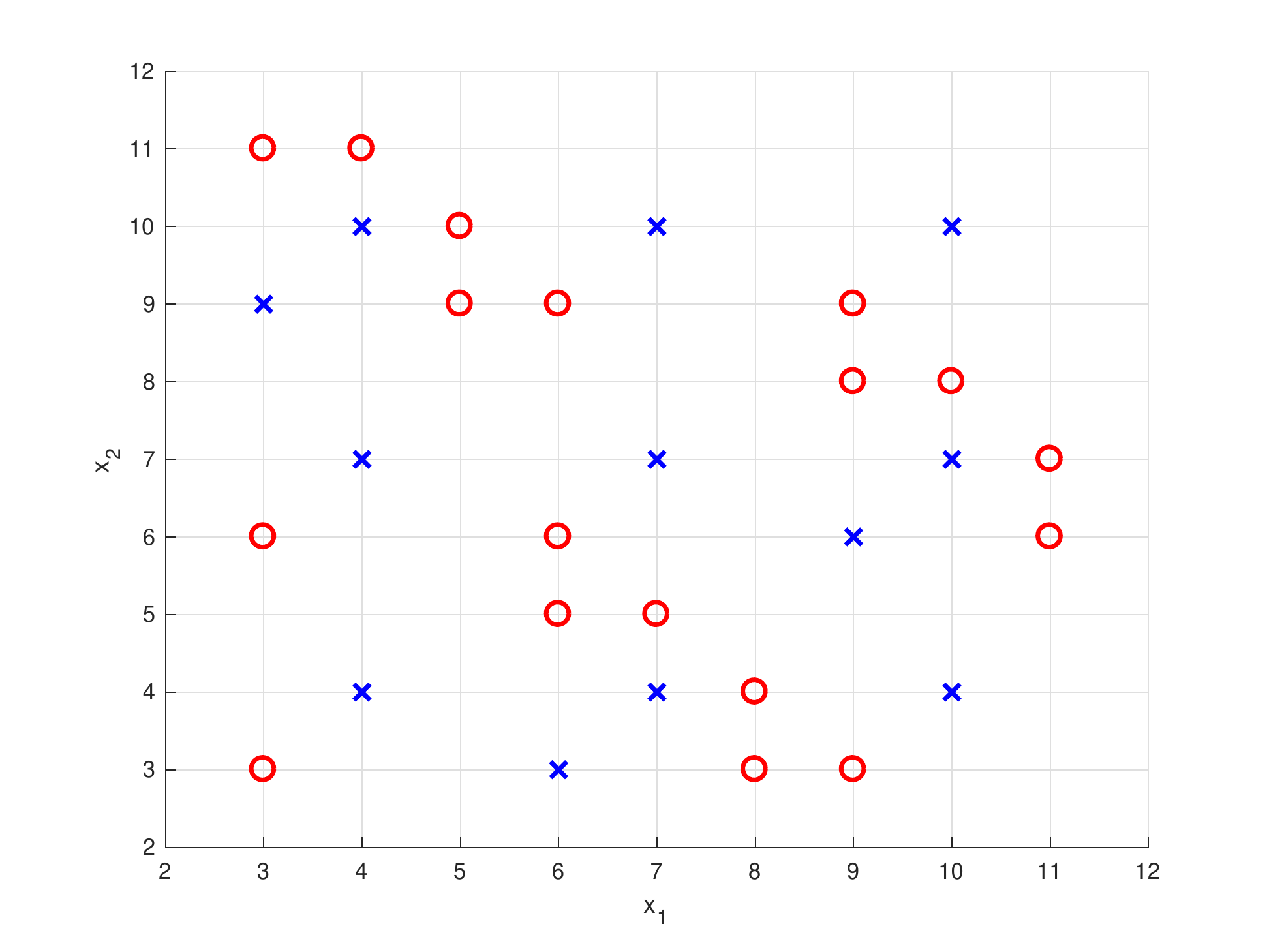}
		\caption{Dataset $\TG$ constructed from $G$.}
		\label{graph_bi_example_fig2}
	\end{subfigure}
	\caption{Example of reduction from balanced bipartite graph $G$ to dataset $\TG$.}
	\label{graph_bi_example}
\end{figure}

{\bf Example.} Consider the example balanced bipartite graph $G = (V1,V2,E)$ shown in Figure~\ref{graph_bi_example_fig1}, where $|V_1| = |V_2| = 3$ and $E = \{(1,1),(1,2),(2,2),(2,3),(3,1),(3,3)\}$. The constructed dataset $\TG$ using the construction defined above is shown in Figure~\ref{graph_bi_example_fig2}.

Now, we would like to show that in the constructed dataset $\mcal{T}_G$: finding the minimum $R_1 = R_2$ such that the $\mcal{L}(\mcal{E},\mcal{D},\mcal{T}_G) = 0$ is equivalent to finding the size of the maximum BCBS in $G$. 

To do so, we first discuss the following properties of the gadget datasets $\mcal{T}_{v_1,v_2}^{(1)}$ and $\mcal{T}_{v_1,v_2}^{(2)}$ and the dataset $\mcal{T}_G$ below.

\noindent{\bf Properties of $\TG$, $\TV^{(1)}$ and $\TV^{(0)}$.}
\begin{enumerate}
	\item Points in the dataset $\TG$ are arranged in {\bf rows} and {\bf columns} in $\mbb{R}^2$ indexed by
	\[
	{\rm rows} = \bigcup_{v_2 \in [1:|V_2|]} \{3v_2, 3v_2+1,3v_2+2\}, \qquad {\rm cols} = \bigcup_{v_1 \in [1:|V_1|]} \{3v_1, 3v_1+1,3v_1+2\}
	\]\label{prop1}
	
	\item Since~\ref{iii} assumes assumes $\mcal{E}_k^{-1}(\bz_k)$ is a single $\mathcal{X}_{\Omega_k}$-bin, then quantization of $x_1$ (resp., $x_2$) in $\TG$ is equivalent to combining adjacent columns (resp., rows) in $\TG$. Combining non-adjacent columns (resp., rows) is not allowed as it will not result in a single bin as assumed by \ref{iii}. We say that two columns (resp., rows) are {\bf combined}, if they are assigned to the same $\bz_1$ (resp., $\bz_2$);\label{prop2}
	
	\item For any $v_2 \in [1:|V_2|]$, there exists a $v_1$ such that $\mcal{T}_{v_1,v_2} = \mcal{T}_{v_1,v_2}^{(1)}$, i.e., is of type 1. This is due to the fact that the vertex indexed by $v_2$ has a non-zero degree.
	A similar property exists, flipping the roles of $v_1$ and $v_2$ in the statement above;\label{prop3}
	
	\item Observing the structure of the gadget dataset $\TV^{(1)}$ and $\TV^{(0)}$ in Figure~\ref{T_v_gadgets}, it is not difficult to see that only the columns (resp., rows) corresponding to indices $x_1 = 3v_1,\ 3v_1+1$ (resp., $x_2 = 3v_2,\ 3v_2+1$ can be combined while maintaining $\mcal{L}(\mcal{E},\mcal{D},\TV) = 0$. For shorthand, we refer to 
	action of combining these columns (resp., rows) with $\mb{q}_{v_1}$ (resp., $\mb{q}_{v_2}$), where
	\begin{align}\label{eq:quant_actions}
	\qa &\text{: combining the columns $x_1 = 3v_1,\ x_1 = 3v_1+1$}, \nonumber \\
	\qb &\text{: combining the rows $x_2 = 3v_2,\ x_2 = 3v_2+1$}.
	\end{align}\label{prop4}
	
	\item The dataset types $\TV^{(1)}$ and $\TV^{(0)}$ impose \underline{allowance} and \underline{restriction} relations on the possible quantizations (combining of columns/rows). 
	It is not hard to see from Figure~\ref{T_v_gadgets} that in $\TV^{(1)}$, we can apply both combining actions $\qa$ and $\qb$ simultaneously without incurring any penalty in the misclassification loss. On the other hand, in $\TV^{(0)}$, we cannot apply both combining actions $\qa$ and $\qb$ without incurring a misclassification loss (two points from different class will collapse to the same position and become indistinguishable after quantization). Thus we have the following equivalent relations  in $G$ and $\TG$ {\bf [This is the key property in the reduction]}\label{prop5}
	\begin{align}
	(v_1,v_2) \not\in E \iff \TV^{(0)} \subseteq \TG \iff \qa, \qb\ \text{are mutually exclusive},
	\end{align}
	and similarly
	\begin{align}
	(v_1,v_2) \in E \iff \TV^{(1)} \subseteq \TG \iff \qa, \qb\ \text{are not mutually exclusive}.
	\end{align}	
	\item Property~\ref{prop3} and~\ref{prop4} together imply that in $\TG$, columns (resp., rows) from different $\TV$ cannot be combined together while maintaining $\mcal{L}(\mcal{E},\mcal{D},\TV) = 0$. In other words, the only available quantization design actions that maintain zero loss are those defined by $\qa$ and $\qb$ in~\eqref{eq:quant_actions};\label{prop6}
	\item From Property~\ref{prop3}, we have that the total number of columns (resp., rows) in the $\TG$ is equal to $3|V_1|$ (resp., 3$|V_2|$).\label{prop7}
	
	\item From Property~\ref{prop7}, by default, we would need $2^{R_1} = 3|V_1|$ and $2^{R_2} = 3|V_2| = 3|V_1|$ to represent the data without quantization (recall that $|V_1| = |V_2|$ as $G$ is balanced).\label{prop8}
\end{enumerate}

With the previously observed properties, we can now state the equivalence between finding maximum BCBS and finding the minimum $R = R_1 = R_2$ that achieves $\mcal{L}(\mcal{E},\mcal{D},\TG) = 0$.
In particular, we have that 

\newtheorem{claim}{Claim}
\begin{claim} \label{claim_equiv_P3}
	{\rm 
		Any quantization (combining) approach in $\TG$ with $2^R = 3|V_1| - W$ quantization bins for both $x_1$ and $x_2$ such that $\mcal{L}(\mcal{E},\mcal{D},\TV) = 0$ results in a BCBS in $G$ of size $W$ and vice versa. 
	}
\end{claim}
By this relation between the number of bits $R$ and the size of bi-clique $W$, we have that finding the maximum BCBS size in $G$ is equivalent to finding the minimum $R$ bits that results in $\mcal{L}(\mcal{E},\mcal{D},\TV) = 0$.

What remains is to prove both directions in Claim~\ref{claim_equiv_P3}.

\noindent \underline{BCBS $\to$ quantization:} Assume that the maximum BCBS is equal to $W$. Then there exists a set $W_1 \subseteq [1:|V_1|]$ and $W_2 \subseteq [1:|V_1|]$ such that: (i) $|W_1| = |W_2| = W$; (ii) $\forall (w_1,w_2) \in W_1 \times W_2$, we have that $(w_1,w_2) \in E$ and $\mcal{T}^{(1)}_{w_1,w_2} \subseteq \TG$. As a result, we can apply the combining actions $\mb{q}_{w_1} \forall w_1 \in W_1$ and $\mb{q}_{w_2} \forall w_2 \in W_2$ while still maintaining zero loss as none of them are in conflict with one another (recall~\eqref{claim_equiv_P3}). Thus, after applying these combinations, we are left with $3|V_1| - W$ columns and rows, which is our intended value for $2^R$.

\noindent \underline{quantization $\to$ BCBS:} 
Assume that we have a zero loss quantization that uses $2^R = 3|V_1| - W$ bins for either $x_1$ and $x_2$.
Since, the possible combining actions are those given by~\eqref{eq:quant_actions}, then there exits a set $W_1, W_2 \subseteq [1:|V_1|]$, such that the given quantization is designed by applying $\mb{q}_{w_1}$ and $\mb{q}_{w_2}\forall \ w_1\in W_1, w_2\in W_2$. Since, applying the aforementioned combining actions did not result in an increase in the loss, then by Property~\ref{prop4}, we have that 
\[	
\forall (w_1,w_2) \in W_1,W_2, (w_1,w_2) \in E.
\]
Thus, we have that $W_1$ and $W_2$ index vertices in $G$ that form a balanced bi-clique of size $W$.

This concludes the proof for the case, where all vertices in $G$ have a non-zero degree.

Finally, let us assume that our bipartite graph $G$ has some vertices with zero degrees. In particular let $G = G' \cup \widehat{G}$, where $G'=  (V'_1,V'_2,E)$ is the subgraph with all vertices with non-zero degrees and $\widehat{G} = (\widehat{V}_1,\widehat{V}_2,\phi)$ be the bipartite graph combining all vertices with zero degrees. 

Here, we can follow the same logic for $G'$ to construct $\mcal{T}_{G'}$. In addition, for $\widehat{G}$, we can add an additional dataset $\mcal{T}_{\widehat{G}}$ (shown in Figure~\ref{fig:P3_consolidation}) such that 
$\mcal{T}_{\widehat{G}}$ is a grid of points of size $3|\widehat{V}_1| \times 3|\widehat{V}_2|$ that cannot be quantized further with zero loss.
$\TG = \mcal{T}_{G'} \cup \mcal{T}_{\widehat{G}}$ as seen  in Figure~\ref{fig:P3_consolidation}.

\begin{figure}[!ht]
	\centering
	\includegraphics[width=0.7\textwidth]{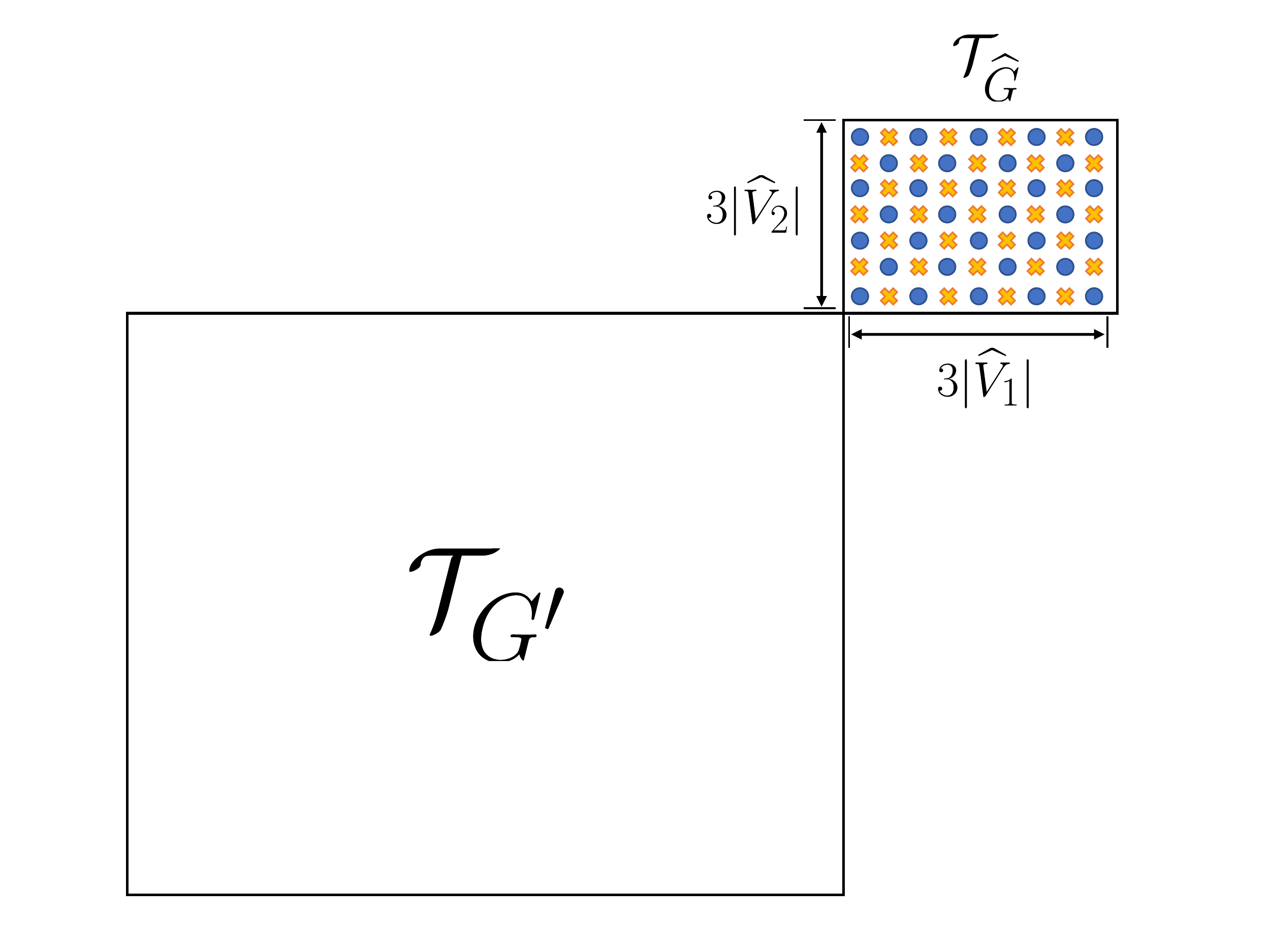}
	\caption{Illustration of $\TG$ decomposition into $\mcal{T}_{\widehat{G}}$ and $\mcal{T}_{G'}$.}
	\label{fig:P3_consolidation}
\end{figure}
It is not difficult to see that Claim~\ref{claim_equiv_P3} extends to this case as well. This concludes the proof of NP-hardness of~\ref{iii}.

\subsection{NP-hardness of \ref{iv}}

To prove \ref{iv}, we notice the following. Consider a training dataset with $n=3$ features, where the training data points have arbitrary values for features $x_1,x_2$, while feature $x_3$ takes values $x_3=-1$ for all points that belong to class $0$ and $x_3=1$ for all points that belong to class $1$. Such training data points are linearly separable by the hyperplane $x_3=0$. However, if $R_3=0$, we cannot send any information based on feature $3$, hence, we have to only consider the projection of the training data on features $x_1,x_2$ ignoring feature $x_3$. This shows that problem \ref{iii} is a special case of problem \ref{iv}. For instance, we can have a training dataset that have values for features $x_1,x_2$ generated based on the method used to prove \ref{iii}, while feature $x_3$ takes $x_3=-1$ for all points that belong to class $0$ and $x_3=1$ for all points that belong to class $1$.  It follows that problem \ref{iv} is NP-hard. 

\subsection{Hardness of approximation}

\par From the proof of \ref{i},\ref{ii}, we can see that a polynomial-time algorithm that approximates $2^{R_1}$ for problem \ref{i} or \ref{ii} within $O(N^{1-\epsilon})$ for some $\epsilon>0$ can be used to approximate the chromatic number within $O(N^{1-\epsilon})$ in polynomial time, since the chromatic number problem is polynomial-time reducible to problem \ref{i} or \ref{ii} with $\mathcal{X}(G)=\min 2^{R_1}$, and number of training points equal to the number of vertices in the graph. However, it was shown in \cite{zuckerman2006linear} that approximating the chromatic number within $O(N^{1-\epsilon})$ is NP-hard $\forall \epsilon>0$.

\par Similarly, from the proof of \ref{iii},\ref{iv}, we can see that a polynomial-time algorithm that approximates $2^{R}$ for problem \ref{iii} or \ref{iv} within a factor of $O(N^{\frac{1}{2}-\epsilon})$ for some $\epsilon>0$ can be used to approximate the BCBS within a factor of $O(N^{1-\epsilon})$ in polynomial time. The fraction $\frac{1}{2}$ in the exponent is because the BCBS problem is polynomial-time reducible to problem \ref{iii} or \ref{iv} with number of training points in the order of the square of the number of vertices in the graph. It was shown in \cite{manurangsi2017inapproximability} that approximating the BCBS within $O(N^{1-\epsilon})$ is NP-hard $\forall \epsilon>0$ assuming the \textit{Small Set Expansion Hypothesis} (SSEH) and that NP $\nsubseteq$ BPP.

\section{On-the-line optimality} \label{proof:on-the-line}
In this appendix, we prove the optimality of Algorithm \ref{alg3} under the considered restriction (horizontal and vertical lines defining $d_{k,i}$ meet along the line $x_1 = x_2$) and assuming that the data points are linearly separable and scaled such that the line $x_1=x_2$ separates the data. 

We prove by induction that $\forall k \in[1:2N]$: at the $k$-th iteration of the algorithm, it finds the optimal quantization boundaries considering only the points $\mcal{T}_{s_k}  = \{(\mb{x},y(\mb{x})) \in \mcal{T} | x_1, x_2 \leq s_k \},\ \forall b \in [1:2^R-1]$, where $s_k$ is the $k$-th smallest element in the set of possible boundaries $\mathbf{s}$.

\begin{itemize}
	\item At the first iteration, there is only one possible position for all the boundaries, $s_1$. Hence, at the first iteration the algorithm finds the optimal quantization boundaries considering the points $\mcal{T}_{s_k}$, $\forall b \in [1:2^R-1]$.
	\item Assuming that at iteration $k$ the algorithm finds the optimal quantization boundaries considering the points $\mcal{T}_{s_k}$, we show that it finds the optimal quantization boundaries considering the points $\mcal{T}_{s_{k+1}}$ at iteration $k+1$, $\forall b \in [1:2^R-1]$.

	At iteration $k+1$, the only possible positions for the first boundary before $s_{k+1}$ are the $k$ positions $s_1,s_2,...,s_k$. Therefore, to find the optimal boundaries at iteration $k+1$, we can condition on the position of the first boundary before $s_{k+1}$, and then optimize over this position. Conditioned on the position of the first boundary before $s_{k+1}$ being at $s_\ell$, the updated loss function can be expressed as
	\begin{equation}
	\mcal{L}_{B\cup s_\ell}(\mathbf{\mathbf{\mathcal{E}}},\mcal{T}_{s_{k+1}})=\mcal{L}_{B}(\mathbf{\mathbf{\mathcal{E}}},\mcal{T}_{s_{\ell}})+\displaystyle\min_{c \in \{1,2\}} |\{j | s_\ell \prec \mb{x}^{(j)} \preceq s_{k+1},\ y^{(j)} = c\}|,
	\end{equation}
	where $\mcal{L}_{B}(\mathbf{\mathbf{\mathcal{E}}},\mcal{T}_{s_{\ell}})$ is the loss $\mcal{L}(\mathbf{\mathbf{\mathcal{E}}},\mathcal{D},\mcal{T}_{s_{\ell}})$ when using the boundaries in the set $B$, and $B$ is the set of boundaries in the region of the space defined by $\mathbf{x}\in \mathbb{R}^2:\mathbf{x}\preceq s_\ell$, $|B\cup s_\ell |=b$. This is, as we discussed in the paper, due to the fact that the misclassified points contributing to the quantizer loss can only lie in the $2$-dimensional intervals crossed by the line $x_1 = x_2$. Which is because any other interval lies on one side of the line that separates the points, hence, contains points from only one class. Minimizing over all the possible values for $s_\ell , B$ we get
	\begin{equation}
	E(s_{k+1},b) = \min_{B,\ell : |B|\leq b-1, \ell <k+1} \{ \mcal{L}_{B}(\mathbf{\mathbf{\mathcal{E}}},\mcal{T}_{s_{\ell}})+\displaystyle\min_{c \in \{1,2\}} |\{j | s_\ell \prec \mb{x}^{(j)} \preceq s_{k+1},\ y^{(j)} = c\}|\}.
	\end{equation}
	We can observe that only the first term in the previous minimization depends on $B$, hence, we can optimize over $B$ first, which gives
	\begin{equation}
	E(s_{k+1},b) = \displaystyle\min_{\ell < k+1}\left\{E(s_\ell,b-1) + \displaystyle\min_{c \in \{1,2\}} |\{j | s_\ell \prec \mb{x}^{(j)} \preceq s_{k+1},\ y^{(j)} = c\}| \right\},
	\end{equation}
	which is the update rule used in the algorithm. Hence, the boundaries corresponding to $E(s_{\ell ^*},b-1)$ along with the boundary at $s_{\ell ^*}$ are the optimal boundaries at iteration $k+1$, where 
	\begin{equation}
	\ell ^*=\displaystyle\argmin_{\ell < k+1}\left\{E(s_\ell,b-1) + \displaystyle\min_{c \in \{1,2\}} |\{j | s_\ell \prec \mb{x}^{(j)} \preceq s_{k+1},\ y^{(j)} = c\}| \right\}.
	\end{equation}
\end{itemize}

\section{Discussion on breaking ties in {\tt GBI} : the purity criterion}
\label{sec:purity_criterion}
In this section we illustrate what we call the purity criterion which is used to break ties in $\texttt{GBI}$. If it happens that two or more possible boundaries lead to the same quantizer loss (something that happened surprisingly often in our experiments), it makes a significant performance difference to add the boundary that looks ahead to allow future boundaries to further decrease the loss. The intuition behind this is the following: for a tie, we have a fixed number of misclassified points; what matters for the algorithm performance is that the misclassified points are in bins that can be more easily partitioned to bins that contains no misclassification in a next iteration. This was more likely to happen in our experiments (and small examples) if a bin that has misclassified points contained a number of majority class points that was just slightly higher than that of the misclassified classes.
 Formally, let $H$ be an $\mathbb{R}^n$-bin and let $\mcal{T}$ be the set of training points $\{(\mb{x}^{(i)}, y(\mb{x}^{(i)})\}$. Define, $B(H,\mcal{T})$ as
 \[
 B(H,\mcal{T}) := 
 \begin{cases}
 0 & \text{if all points in $H$ are of the same class}, \\
 \displaystyle\max_{c \in \mcal{C}}\left|\{(\mb{x},y(\bx)) {\in} \mcal{T} {\cap} H\ \text{s.t.} \ y(\mb{x}) {=} c \}\right| & \text{otherwise}.
 \end{cases}
 \]
 $B(H,\mcal{T})$ counts the number of points of the {\bf majority} class in $H$ when there is at least two or more classes represented in $H$, and is zero otherwise.
 For a particular set of boundaries that partition $\mbb{R}^n$ into the $\mathbb{R}^n$-bins $\{H_k\}_{k=1}^M$, we want to {\bf minimize} the {\bf purity criterion}  defined as  
 \[
 U(\{H_k\}_{k=1}^M, \mcal{T}) := \frac{\sum_{k=1}^M {B}(H_k,\mcal{T})^2}{N},
 \]
 where the term ${B}(H_k,\mcal{T})^2$ penalizes bins with more majority points. $U(\{H_k\}_{k=1}^M, \mcal{T})$ is minimized when the correctly classified points represent a weak majority in the bins that have misclassification. This allows for the bins that have misclassification to be easily partitioned to bins that have no misclassification in a following iteration. For illustration, consider the example shown in Figure~\ref{fig_purity}.
In this example all the potential boundaries \circled{1}, \circled{2}, \circled{3} ,\circled{4} result in the same value for the quantizer loss. However, it is clear that, unlike boundaries \circled{3},\circled{4}, if we pick boundaries \circled{1}, \circled{2}, this allows for the separation of the ``o'' points (red) from the ``x'' points (blue) in the next iteration.
The purity criterion chooses boundary \circled{2} as shown by the values in Figure~\ref{fig_purity}. This choice separates the maximum number of ``x'' points from the bin that have misclassification.
Note that, a power greater than $1$, hence a function with slope that increases when we move away from zero, is needed in the purity criterion to highly penalize the bins that have misclassification having high concentration of majority class points, which prevents isolating the misclassified points in the following iterations. If we use a power of $1$, the purity criterion is reduced to the quantizer loss. In our case we use a power of $2$. 
\begin{figure}[!ht]
\centering
	\includegraphics[width=0.8\textwidth]{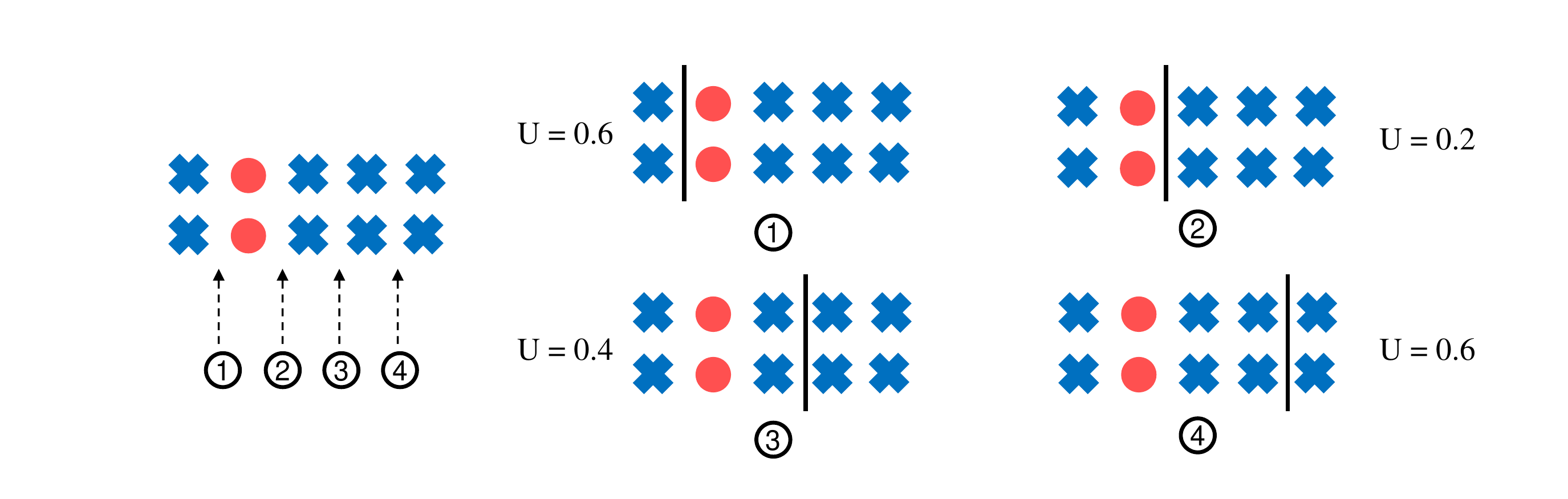}
	\vspace{-1em}
	\caption{Illustration of the purity criterion.}
	\label{fig_purity}
\end{figure}

\section{Parameters of trained models for experimental evaluation}
\label{appendix:net_structure}
In this appendix, we describe the structure of the encoders/decoders neural networks and parameters used in the experimental results for each dataset.
\subsection{\bf sEMG dataset:}
The structures of the neural networks used by the encoders $f_k(\cdot;\theta_k)$ and decoder $g(\cdot;\phi)$ are given in Table~\ref{enc_sEMG}.

{\bf $\bullet\ $ General parameters.} The distributed quantization system was trained using Adam optimizer with learning rate $10^{-3}$ for $300$ epochs.
The parameters of the pretrained classifier $\mcal{C}$ were frozen (not updated) during the learning phase. 

\begin{table}[h]
	\caption{Structure of the encoder neural networks $f_k(\cdot;\theta_k)$ and the decoder neural network $g(\cdot;\phi)$ used for the sEMG dataset.}
	\label{enc_sEMG}
	\centering
	\begin{tabular}{c|c|c}
		\multicolumn{3}{c}{\bf Encoder neural network $f_k(\cdot;\theta_k),\ \forall k \in [1:4]$} \\
		\toprule
		Layer Index & Layer Type & Output size \\
		\midrule
		\midrule
		1 & FC-Relu         & 90\\
		\hline
		2 & FC-Relu     & 170\\
		\hline
		3 & FC-Batchnorm-Tanh         & \# of bits per encoder\\
		\bottomrule
	\end{tabular}
	\hspace{2em}
	\begin{tabular}{c|c|c}
		\multicolumn{3}{c}{\bf Decoder neural network $g(\cdot;\phi)$} \\	
		\toprule
		Layer Index & Layer Type & Output size \\
		\midrule
		\midrule
		1 & FC-Relu         & 170\\
		\hline
		2 & FC-Relu       & 90\\
		\hline
		3 & FC         &	8\\
		\bottomrule
	\end{tabular}
\end{table}

{\bf $\bullet\ $ NN-REG.} For our approach in Section~\ref{sec:neural_REG}, we chose the regularization parameter $\beta$ through $5$-fold cross validation out of possible parameter values $\{0,0.1,\cdots,2\}$. The chosen regularization parameter is $\beta = 1.4$. 

{\bf $\bullet\ $ NN-GBI.} For our approach in Section~\ref{sec:neural_GBI}, the output of the last layer in the encoder was chosen to be $1$ when applying {\bf Phase 1} (refer to Section~\ref{sec:neural_GBI}). When applying $\tt GBI$, a batch size of 300 data points was used in each iteration of the $\tt GBI$ algorithm. After designing the quantizer with \texttt{GBI}, the network is trained for an extra $20$ epochs with the quantizer in {\bf Phase 3}.

\subsection{\bf CIFAR10 dataset:}
The structures of the neural networks used by the encoders $f_k(\cdot;\theta_k)$ and decoder $g(\cdot;\phi)$ are given in Table~\ref{enc_CIFAR} and Table~\ref{dec_CIFAR}, respectively. 

{\bf $\bullet\ $ General parameters.} The distributed quantization system was trained using Adam optimizer with learning rate $10^{-3}$ for $200$ epochs.
The parameters of the pretrained classifier $\mcal{C}$ were frozen (not updated) during the learning phase. 

{\bf $\bullet\ $ NN-REG.} The chosen regularization parameter is $\beta = 0.25$. 

{\bf $\bullet\ $ NN-GBI.} For our approach in Section~\ref{sec:neural_GBI}, the output of the last layer in the encoder was chosen to be $5$ when applying {\bf Phase 1} (refer to Section~\ref{sec:neural_GBI}). When applying $\tt GBI$, a batch size of 300 data points was used in each iteration of the $\tt GBI$ algorithm. After designing the quantizer with \texttt{GBI}, the network is trained for an extra $50$ epochs with the quantizer in {\bf Phase 3}.

\begin{table}[h]
	\caption{Structure of encoder neural networks $f_k(\cdot;\theta_k)$ used for CIFAR-10 dataset.}
	\label{enc_CIFAR}
	\centering
	\begin{tabular}{c|c|c|c|c|c|c|c}
		\toprule
		\makecell{Layer \\ Index} & \makecell{Layer \\ Type} & \makecell{Output \\ size} & \makecell{Input \\ channels} & \makecell{Output \\ channels} & \makecell{Kernel \\ size} &  Stride & Padding\\
		\midrule
		\midrule
		1 & Conv-Relu         & - & 3 & 64 & 3 & 1 & 1 \\
		\hline
		2 & Conv-Relu         & - & 64 & 64 & 3 & 1 & 1 \\
		\hline
		3 & Maxpool         & - & - & - & 2 & 2 & - \\
		\hline
		4 & Conv-Relu         & - & 64 & 128 & 3 & 1 & 1 \\
		\hline
		5 & Conv-Relu         & - & 128 & 128 & 3 & 1 & 1 \\
		\hline
		6 & Maxpool         & - & - & - & 2 & 2 & - \\
		\hline
		7 & FC-Tanh         & number of bits per encoder & - & - & - & - & - \\
		\bottomrule
	\end{tabular}
\end{table}

\begin{table}[h]
	\caption{Structure of decoder neural network $g(\cdot;\phi)$ used for CIFAR-10 dataset.}
	\label{dec_CIFAR}
	\centering
	\begin{tabular}{c|c|c|c|c|c|c|c}
		\toprule
		\makecell{Layer \\ Index} & \makecell{Layer \\ Type} & \makecell{Output \\ size} & \makecell{Input \\ channels} & \makecell{Output \\ channels} & \makecell{Kernel \\ size} &  Stride & Padding\\
		\midrule
		\midrule
		1 & FC-Relu         & 80 & - & - & - & - & - \\
		\hline
		2 & ConvTranspose-Relu-Batchnorm         & - & 5 & 5 & 4 & 2 & 1 \\
		\hline
		3 & ConvTranspose-Relu-Batchnorm         & - & 5 & 5 & 4 & 2 & 1 \\
		\hline
		4 & ConvTranspose-Relu-Batchnorm        & - & 5 & 5 & 4 & 2 & 1 \\
		\hline
		5 & Conv         & - & 5 & 3 & 5 & 1 & 2 \\
		\bottomrule
	\end{tabular}
\end{table}

\newpage
\bibliographystyle{IEEEtran}
\bibliography{QUANT_JSAIT}

\end{document}